\documentclass[twoside]{article}

\usepackage[numbers]{natbib}
\usepackage{booktabs}
\usepackage[colorlinks=true, linkcolor=blue, citecolor=blue, urlcolor=blue]{hyperref}

\usepackage[preprint]{icml2026}

\usepackage{amsfonts, amsmath, amssymb, amsthm}
\usepackage{algorithm}
\usepackage[noend]{algpseudocode} 
\usepackage{changepage}
\usepackage{multirow}
\usepackage[most]{tcolorbox} 

\usepackage{graphicx}
\usepackage{xcolor}
\definecolor{tabblue}{RGB}{0,102,204}

\hypersetup{
    colorlinks=true,
    linkcolor=tabblue,
    citecolor=tabblue,
    urlcolor=tabblue,
}

\newcommand{\chan}[1]{\textcolor{black}{#1}}

\usepackage[capitalize,nameinlink,noabbrev]{cleveref} 

\theoremstyle{plain}
\newtheorem{theorem}{Theorem}[section]
\newtheorem{proposition}{Proposition}[section]   
\newtheorem{lemma}{Lemma}[section]             

\theoremstyle{definition}
\newtheorem{definition}{Definition}[section]

\theoremstyle{remark}
\newtheorem{remark}{Remark}[section]

\newcommand{\bt}[1]{\mbox{$\bf #1$}}

\newcommand{\baseL}{\bt L}
\newcommand{\newL}{\tilde{\bt L}}
\newcommand{\newU}{\tilde{\bt U}}

\newcommand{\bfnew}{\tilde{\pmb{\lambda}}}
\newcommand{\bfbase}{\pmb{\lambda}}

\newcommand{\Dml}{\bt D(\bfnew, \bfbase)}

\newcommand{\diag}[1]{\mathrm{diag}(#1)}

\newcommand{\defeq}{\doteq}

\definecolor{tabblue}{RGB}{0, 114, 178}
\definecolor{taborange}{RGB}{255, 127, 14}
\definecolor{tabpurple}{RGB}{148, 103, 189}

\icmltitlerunning{L2G-Net: Local to Global GNNs via Cauchy Factorizations}

\begin{document}

\renewcommand{\algorithmicrequire}{\textbf{Input:}}
\renewcommand{\algorithmicensure}{\textbf{Output:}}

\twocolumn[

\icmltitle{L2G-Net: Local to Global Spectral Graph Neural Networks \\ via Cauchy Factorizations}

\begin{icmlauthorlist}
    \icmlauthor{Samuel Fern\'andez-Mendui\~na}{yyy}
    \icmlauthor{Eduardo Pavez}{yyy}
    \icmlauthor{Antonio Ortega}{yyy}
\end{icmlauthorlist}

\icmlaffiliation{yyy}{University of Southern California}

\icmlcorrespondingauthor{Samuel Fernández-Menduiña}{samuelf9@usc.edu}

  \icmlkeywords{Machine Learning, ICML}

  \vskip 0.3in

]

\printAffiliationsAndNotice{}  

\begin{abstract}
Despite their theoretical advantages, spectral methods based on the graph Fourier transform (GFT) are seldom used in graph neural networks (GNNs) due to the cost of computing the eigenbasis and the lack of vertex-domain locality in the resulting representations. As a result, most GNNs rely on local approximations such as polynomial Laplacian filters or message passing, which limit their ability to model long-range dependencies. In this paper, we introduce an exact factorization of the GFT into operators acting on subgraphs, which are then combined via a sequence of Cauchy matrices. Building on this factorization, we propose a new class of spectral GNNs, termed L2G-Net (Local to Global Net). Unlike existing spectral methods, which are either fully global (when using the GFT) or local (when using polynomial filters), L2G-Net operates by processing the spectral representations of subgraphs and then combining them via structured matrices. Our algorithm avoids full eigendecompositions, exploiting graph topology to construct the factorization with quadratic complexity in the number of nodes, scaled by the maximum cut size between subgraphs. Experiments stressing long-range dependencies on large graphs show that L2G-Net scales to regimes out of reach for the standard GFT, and is competitive with state-of-the-art methods with orders of magnitude fewer learnable parameters.
\end{abstract}

\section{Introduction}
\chan{Spectral graph neural networks \cite{bruna2014spectral} process graph signals by projecting them onto the eigenbasis of the graph Laplacian. This operation,} also known as the graph Fourier transform (GFT) \cite{ortega2018graph}, is analogous to the Fourier transform in signal processing in regular domains, yielding a representation that captures the global structure of the graph \cite{chung1997spectral}.

Despite using well-defined graph signal frequencies, 
\chan{GFT-based} GNNs are far less used in practice than polynomial approximations \cite{defferrard2016convolutional, levie2018cayleynets, he2021bernnet} or message-passing neural networks (MPNNs) \cite{kipf2017semi}. Two limitations drive this lack of adoption: the cubic complexity of computing the GFT makes it impractical for large graphs \cite{Bronstein2017geometric}, and GFT-domain operations are global, i.e.,  modifying a single spectral coefficient affects all nodes simultaneously \cite{ortega2018graph}. 
While globality is well-suited to capture long-range interactions \cite{dwivedi2022long}, \chan{the GFT lacks the local inductive biases that benefit short-range graph learning tasks, which dominate many of the existing benchmarks \cite{alon2020bottleneck}}.

\begin{figure*}[t]
    \centering
    \includegraphics[width=\linewidth]{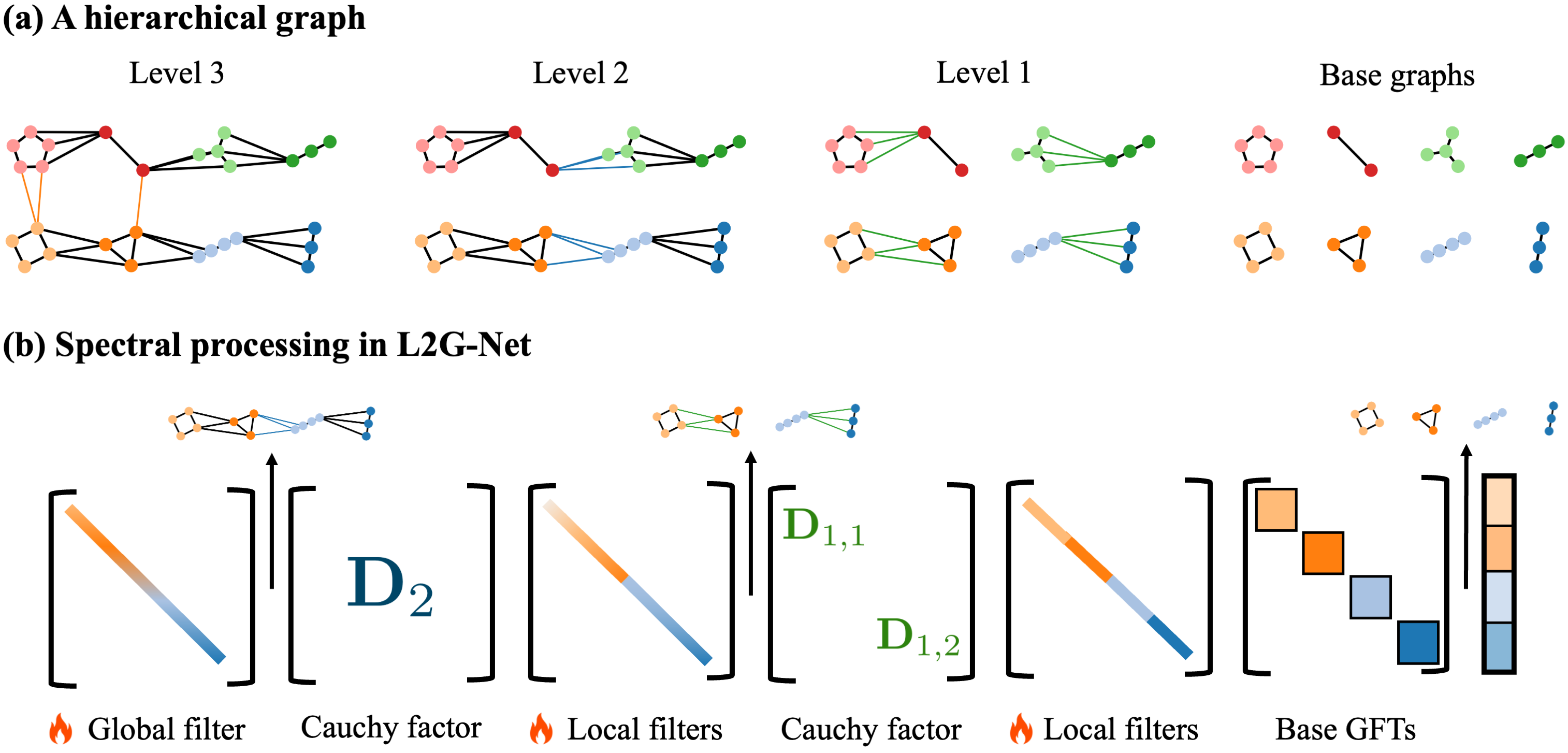}
    \caption{(a) A graph in $\mathcal{F}(L=3, \lbrace \mathcal{G}_{i}\rbrace_{i = 1}^8, k=3)$. Bridge edges are shown in different colors. The GFT basis of the whole graph decomposes as the product of the GFT bases of the base graphs by a sequence of Cauchy factors. (b) We first compute the base GFTs of small subgraphs (right) and apply local spectral filters. We mix the outputs via Cauchy factors corresponding to bridge edges. After all merges, a global spectral filter acts on the full graph representation (left), allowing for global filtering while preserving locality.}
    \label{fig:graph_multires}
\end{figure*}

Polynomial approximations and MPNNs replace global GFT operations with repeated  products with a sparse graph operator (e.g., adjacency or Laplacian) \cite{kipf2017semi} or a learned polynomial of the operator \cite{defferrard2016convolutional}. \chan{These approaches are computationally efficient and induce a $k$-hop locality bias, whose alignment with widely-used homophilic benchmarks partially explains their empirical adoption \cite{hariri2025return}}. However, modeling long-range dependencies with local operators requires multiple message-passing steps, often leading to optimization instabilities \cite{arroyo2025vanishing} such as  oversquashing \cite{alon2020bottleneck}. 
Replacing the Laplacian with a polynomial of the Laplacian might alleviate these issues, but often at the cost of restrictive parameter constraints \cite{hariri2025return}. 
Combining global and local information via self-attention  \cite{deng2024polynormer} often sacrifices parameter efficiency and graph-domain interpretability (cf.~\cref{app:related_work}).

In this paper, we introduce a new class of spectral GNNs that builds \textit{global} spectral representations by aggregating \textit{local} spectral coefficients (\cref{fig:graph_multires}). 
We show that any GFT can be \emph{exactly factorized} into a sequence of localized transformations: given a graph partition (\cref{fig:graph_multires} (a)), the GFT admits a hierarchical decomposition into the GFTs of the local \emph{subgraphs} (a connected set of nodes) forming the partition, which are then multiplied by a series of \emph{Cauchy matrices}, each arising from a rank-one update associated with a \emph{cut} edge connecting two subgraphs  
(cf.~\cref{fig:graph_multires}). Matrices in the Cauchy factorization are block-diagonal, hence localized to graph regions, and the composition of all terms yields the global GFT.

\chan{Computing this factorization is efficient: obtaining all terms has quadratic complexity in the number of nodes, scaled by the size of the subgraph cut. 
For graphs with hierarchical structure, common in real-world networks \cite{barthelemy2011spatial, clauset2008hierarchical}, this leads to significant computational savings over the cubic complexity eigendecomposition.} While this complexity result holds for arbitrary graph partitions, the factorization cost may be high for graphs with large cuts. Hence, we introduce a graph decomposition algorithm that explicitly aims to minimize the cost of computing the Cauchy factorization. 
Specifically, we seek balanced partitions with small cuts, directly optimizing the theoretical factorization cost. 
When such favorable partitions are not available, we use spectral sparsification \cite{spielman2008graph} to reduce the cut size while preserving the spectral properties of the original graph. We validate this analysis through experiments on synthetic graphs, confirming the predicted runtime scaling and demonstrating substantial improvements over the dense eigendecomposition of the graph Laplacian.

This divide-and-conquer strategy \chan{is then used as the foundation for a new class of spectral GNNs}, termed Local to Global Net (L2G-Net). 
As shown in \cref{fig:graph_multires}, at each stage, the outputs of each local GFT (i.e., subgraph spectral information) are processed using learnable filters before merging via Cauchy factors. 
In the last level, a global learnable spectral filter processes the entire graph signal without requiring a full eigendecomposition of the Laplacian (cf.~\cref{fig:graph_multires}(b)). This allows L2G-Net to combine global spectral processing with locality biases by construction. 

\chan{From a learning perspective, L2G-Net first analyzes the graph structure in a task-independent manner (by identifying subgraphs and sparsifying), and then builds an architecture tailored to the specific graph and learns its parameters. Thus, different graphs yield different hierarchical structures. This allows local filters to capture patterns within subgraphs, while the global filter models subgraph interactions. This property is relevant in long-range setups. MPNNs and polynomial approximations must relay information through bottlenecks, which leads to oversquashing \cite{alon2020bottleneck}, while graph transformers \cite{dwivedi2020generalization} bypass this with learned attention but require additional inputs (e.g., positional encodings) to encode graph structure. L2G-Net offers a complementary inductive bias by encoding graph structure into the computation (\cref{fig:locality}). We show a positioning table in \cref{tab:method_comparison}; see \cref{app:related_work} for detailed discussions.}

\begin{figure}
    \centering
    \includegraphics[width=0.9\linewidth]{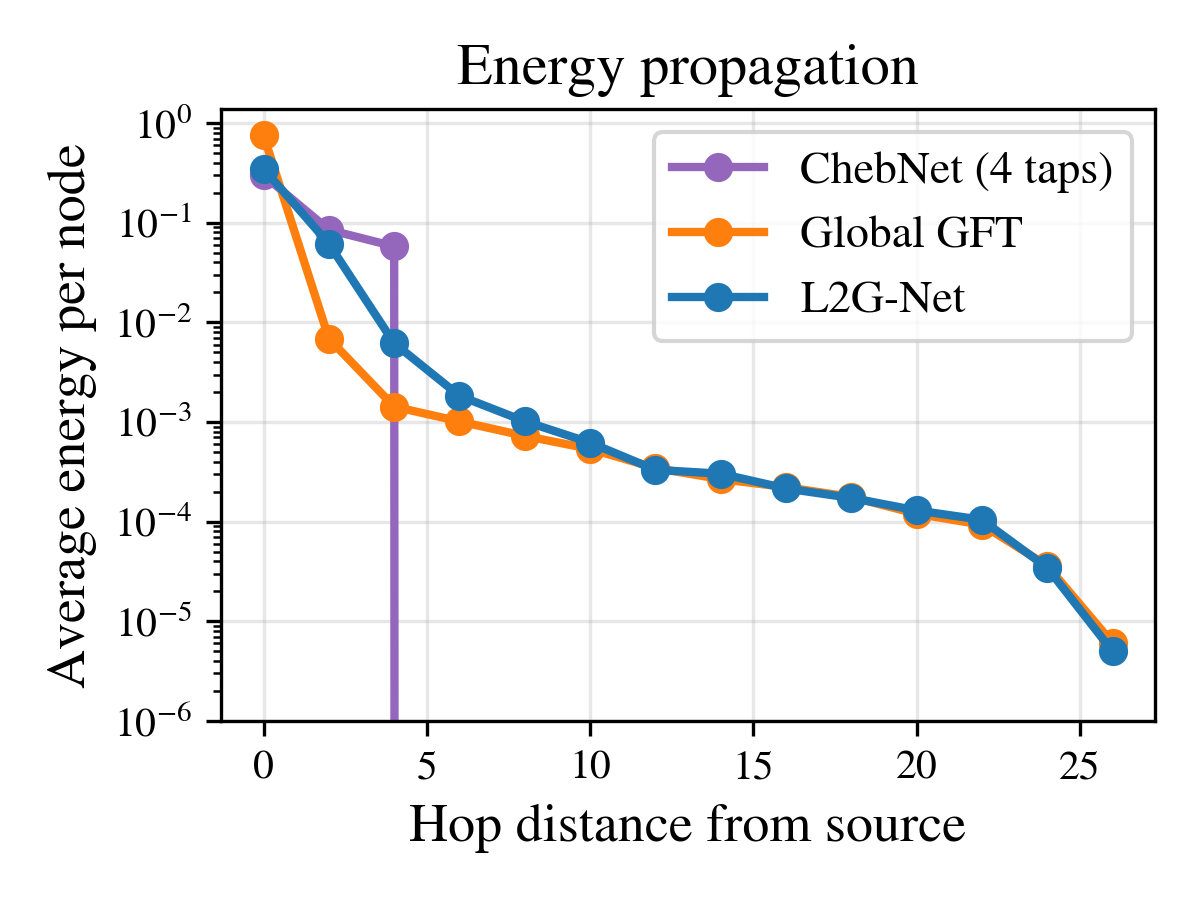}
    \caption{\chan{Average energy versus hop distance from the source when we input a unit impulse at each node of a graph. ChebNet (3rd-order) has a hard cutoff at 3 hops, while a GFT-based GNN propagates energy globally. L2G-Net's locality bias adapts to the graph structure, interpolating between these two extremes.}}
    \vspace{-0.2em}
    \label{fig:locality}
\end{figure}
\begin{table}[t]
\small
\setlength{\tabcolsep}{10pt} 
\centering
\caption{\chan{Comparison of different graph learning methods. CU stands for computational unit and GT for graph transformers.}}
\label{tab:method_comparison}
\begin{tabular}{lccc}
\toprule
\textbf{Method} & \textbf{Cost} & \textbf{Locality} & \textbf{CU} \\
\midrule
MPNNs        &    $O(|E|)$            & Local                     & Node      \\
GT &    $O(n^2)$            & Global   & Node      \\
ChebNet      &     $O(K|E|)$           & $K$-hop                  & Node      \\
GFT          &   $O(n^3)$             & Global          & Graph     \\
L2G-Net & $O(kn^2)$             & Local to global & Subgraph  \\
\bottomrule
\end{tabular}
\end{table}

We apply L2G-Net \chan{to 1) two transductive benchmarks of large-scale graphs \cite{platonov2023critical, liang2025towards}, and 2) an inductive benchmark with smaller graphs \cite{dwivedi2022long}.} The datasets we consider emphasize non-local structural dependencies, where MPNNs are known to struggle due to their inductive biases \cite{alon2020bottleneck, rusch2023survey, hariri2025return}. 
Our results show that L2G-Net outperforms existing spectral strategies \cite{hariri2025return}, with lower complexity than applying the global GFT and a similar number of learned parameters. 
We achieve results competitive with state-of-the-art techniques using attention \cite{deng2024polynormer}, while requiring orders of magnitude fewer learnable parameters.

Our contributions are: 1) showing that the GFT basis admits an exact Cauchy factorization for any partition, computable in quadratic time from the graph Laplacian without full eigendecompositions (\cref{sec:theory}); 2) an algorithm to find partitions that are computationally advantageous for Cauchy factorizations (\cref{sec:finding_hgf}); 3) L2G-Net, a spectral GNN family based on this factorization that combines local and global computations (\cref{sec:cauchy_gnn}).

\section{Background}

\noindent \textbf{Notation.} Uppercase and lowercase bold letters, such as $\bt A$ and $\bt a$, denote matrices and vectors, respectively. The $n$th entry of $\bt a$ is $a_n$, and the $(i, j)$-th entry of $\bt A$ is $A_{ij}$. 

\subsection{Spectral graph theory}
Let $\mathcal{G} = (\mathcal{V}, \mathcal{E}, \bt W, \bt V)$ be a weighted undirected graph, with vertex set $\mathcal{V}$, edge set $\mathcal{E}$, weighted adjacency matrix $\bt W$, and self-loop diagonal matrix $\bt V$. The generalized graph Laplacian (GGL) is $\bt L \defeq \bt D - \bt W + \bt V$, where $\bt D \doteq \diag{\bt 1^\top \bt W}$ is the degree matrix \cite{biyikougu2007laplacian}. Given the eigendecomposition $\bt L = \bt U \diag{\bfbase} \bt U^\top$, we refer to $\bt U^\top$ as the graph Fourier transform (GFT) \chan{basis} \cite{chung1997spectral, ortega2018graph} associated with $\mathcal{G}$. 
Any graph GGL decomposes into a sum of rank-one terms (or ``baby Laplacians''), one for each edge contribution.
\begin{proposition}[\cite{batson2014twice}]
\label{prop:rank_one_laplacians}
    Let $\bt e_j$ be the $j$th canonical vector, for $j = 1, \hdots, n$. Then, the GGL $\bt L$ of an undirected graph can be written as
    \begin{equation}
        \bt L = \sum_{{(i, j)} \in  \mathcal{E}}
        w_{ij} (\bt e_i - \bt e_j)(\bt e_i - \bt e_j)^\top
        + \sum_{(i, i)\in\mathcal{E}}
        v_{ii} \bt e_i\bt e_i^\top,
    \end{equation}
    where $w_{ij}$ denotes the edge weights, for $i, j = 1, \hdots, \vert \mathcal{V} \vert$. 
    \end{proposition}

We use Proposition \ref{prop:rank_one_laplacians} to express the addition of an edge to a graph with Laplacian $\bt L$ as a \textit{rank-one update}:
\begin{equation}
\label{eq:rank-one-update}
  \tilde{\bt L} = \bt L +  w_{ij} \, \bt v \bt v^\top,   
\end{equation}
with $\bt v = (\bt e_i - \bt e_j)$ when adding an edge between $i$ and $j$. 

\subsection{Cauchy matrices}
\begin{definition}[Cauchy matrix \cite{gastinel1960inversion}]
Given two vectors $\bt x \in\mathbb{R}^n$ and $\bt y \in\mathbb{R}^{n}$ with no common entries, the Cauchy matrix $\bar{\bt C}(\bt x, \bt y)\in\mathbb{R}^{n\times n}$ has entries $\bar{C}_{ij} = 1 / ( x_i - y_j )$ for $i, j = 1, \hdots, n$.
\end{definition}

\smallskip

\begin{definition}[Orthogonal Cauchy-like matrix, OCLM \cite{fasino2023orthogonal, cai2018}]
A matrix $\bt C(\bt x, \bt y)\in\mathbb{R}^{n\times n}$ is an orthogonal Cauchy-like matrix if it is orthogonal and can be written as $\bt C(\bt x, \bt y) = \mathrm{diag}(\bt s) \bar{\bt C}(\bt x, \bt y) \mathrm{diag}(\bt t)$ for some $\bt s, \bt t\in\mathbb{R}^{n}$.
\end{definition}

OCLMs relate eigenvectors before and after a rank-one update of a symmetric matrix \cite{fasino2023orthogonal}. Each rank-one update corresponds to an edge in the GGL. 
In particular, for \eqref{eq:rank-one-update}, with $\bt L = \bt U \mathrm{diag}(\pmb{\lambda})\bt U^\top$ and $\tilde{\bt L} = \tilde{\bt U} \mathrm{diag}(\pmb{\tilde{\lambda}})\tilde{\bt U}^\top$, 
it can be shown that \cite{fernandez-menduina2025fast}:
\begin{equation}
\label{eq:icassp}
    \tilde{\bt U}^\top = - \bt C(\tilde{\pmb{\lambda}}, \pmb{\lambda})\bt U^\top.
\end{equation}
For rank-one updates,   $\tilde{\pmb{\lambda}}$ can be found via the secular equation \cite{golub1973some} (cf.~\cref{app:further}). Relation \eqref{eq:icassp} applies only to graphs without repeated eigenvalues and to updates that are not orthogonal to any column of $\bt U$ (cf.~\cref{sup:deflation}).

\subsection{Problem statement}
We seek to construct spectral operators that compute the GFT of a given graph \chan{efficiently via} localized computations. To this end, we rely on hierarchical decompositions into subgraphs connected by cuts of reduced size.

\begin{definition}[Hierarchical graph family (HGF)]
Let $\{\mathcal{G}_i\}_{i=1}^m$ be a collection of base graphs. A graph $\mathcal{G}$ belongs to the hierarchical graph family
$\mathcal{F}(L, \{\mathcal{G}_i\}_{i=1}^m, k)$ if it can be constructed by recursively merging pairs of subgraphs over $L$ levels, where at each level at most $k$ edges (bridge edges) are added between any pair of subgraphs.
\end{definition}

\chan{Every graph belongs to at least one HGF: choose each base graph to be a singleton node and then add all edges between pairs of nodes. HGFs allow us to  formalize the algorithmic complexity of the GFT computation: graphs exhibiting modular or multi-scale structure, which are typical in practice \cite{barthelemy2011spatial}, admit HGFs with small cut sizes, enabling efficient divide-and-conquer processing.}

In this paper, given a graph $\mathcal{G} \in \mathcal{F}(L, \{\mathcal{G}_i\}_{i=1}^m, k)$, we show that its GFT basis can be factorized into a sequence of Cauchy factors (\cref{thm:cauchy_factorization}), and we provide an $O(n^2 k)$ algorithm to compute this factorization (\cref{th:master-eigen}). Moreover, for arbitrary graphs, we propose algorithms to identify hierarchies that optimize factorization complexity  (\cref{sec:finding_hgf}). 
L2G-Net relies on this factorization, yielding a new class of spectral GNNs that avoids $O(n^3)$ eigendecompositions and introduces local inductive biases (\cref{sec:cauchy_gnn}). 

\section{Cauchy factorization of the GFT}
\label{sec:theory}
We first generalize the result of \cite{fernandez-menduina2025fast} to remove the constraints on the eigenvalues and updates (see \cref{sup:deflation} for details) by introducing \textit{Cauchy factors}, which will play a similar role to the Cauchy matrices in \eqref{eq:icassp} for the case of arbitrary graphs. 
While we focus on the Laplacian case, the following results hold for any symmetric matrix.
\begin{definition}[Cauchy factor (CF)]
\label{def:cauchy_factor}
Let $\mathcal{S}$ denote the set of \chan{eigenvalue indices selecting 1) distinct eigenvalues, and among them, 2) those with corresponding eigenvectors not orthogonal to $\bt v$}, and define the symmetric matrix $\tilde{\bt L}  = \bt L + \rho\bt v\bt v^\top$. The associated Cauchy factor is
\begin{equation}
\bt D(\bfnew, \bfbase)
=
\bt P_{\mathcal S}^\top
\begin{bmatrix}
\bt I & \bt 0 \\
\bt 0 & -\bt C(\bfnew_{\mathcal S}, \bfbase_{\mathcal S})
\end{bmatrix}
\bt P_{\mathcal S},
\end{equation}
where $\bt C(\bfnew_{\mathcal S}, \bfbase_{\mathcal S})$ is an OCLM,
$\bt I$ is the identity on invariant spectral directions of dimension $(n - |\mathcal S|)$,
and $\bt P_{\mathcal S}$ permutes the basis so that invariant and affected components are grouped together.
\end{definition}

In this case, $\bt C(\bfnew_{\mathcal{S}}, \bfbase_{\mathcal{S}})$ corresponds to the OCLM 
associated with the deflated problem. The following result generalizes \cite{fernandez-menduina2025fast}
from path-graph Laplacians to arbitrary symmetric matrices.

\begin{lemma}[Progressive decomposition identity]
\label{lemma:main_eq}
Let
\(
\tilde{\bt L} = \bt L + \alpha \bt v \bt v^\top
= \tilde{\bt U}\diag{\tilde{\pmb{\lambda}}}\tilde{\bt U}^\top
\),
with
\(
\bt L = \bt U\diag{\pmb{\lambda}}\bt U^\top
\).
Then, generalizing \eqref{eq:icassp}, the updated eigenvectors are 
\begin{equation}
\label{eq:fwd_map}
\tilde{\bt U}^\top = \bt D(\tilde{\pmb{\lambda}}, \pmb{\lambda}) \bt U^\top.
\end{equation}
\end{lemma}

\begin{proof}
See \cref{thm:main_eq} in the Appendix.
\end{proof}

We can now extend this identity to HGFs.

\begin{theorem}[Cauchy factorization]
\label{thm:cauchy_factorization}
Let
\(
\mathcal{G} \in \mathcal{F}(L, \{\mathcal{G}_i\}_{i=1}^m, k)
\)
with Laplacian
\(
\bt L = \bt U \diag{\pmb{\lambda}} \bt U^\top
\).
Let $\bt U_0$ denote the block-diagonal matrix, where each block contains the eigenvectors
of one subgraph $\mathcal{G}_i$.
Then,
\begin{equation}
\bt U^\top
=
\bt D(\pmb{\lambda}, \tilde{\pmb{\lambda}}_{K-1})
\cdots
\bt D(\tilde{\pmb{\lambda}}_1, \tilde{\pmb{\lambda}}_0)
\bt U_0^\top,
\end{equation}
where $\tilde{\pmb{\lambda}}_0$ are the eigenvalues of the partitioned graph,
and \chan{$\tilde{\pmb{\lambda}}_{K-\ell}$} follows by removing $\ell = 1,\dots,K$ bridge edges from
$\mathcal{G}$, where $K = k(2^L - 1)$.
\end{theorem}

\begin{proof}
Sketch: By \cref{prop:rank_one_laplacians}, the Laplacian of $\mathcal{G}$ can be written
as a sequence of rank-one edge updates to connect the $m$ subgraphs $\mathcal{G}_i$.
Applying the progressive decomposition identity in \cref{lemma:main_eq} to each update
and composing the resulting Cauchy factors yields the stated factorization.
The full proof is given in \cref{app:proof_master}.
\end{proof}
\chan{Intuitively, each Cauchy factor is a mixing step between two subgraphs joined by an edge. An edge addition is a rank-one update of the Laplacian, which rotates the eigenbasis by a Cauchy matrix; since the update only affects the two endpoints' subgraphs, the rotation mixes their spectral coefficients and leaves the rest unchanged. The factorization in \cref{thm:cauchy_factorization} builds the global GFT by chaining these steps: start from independent subgraph GFTs, and merge one cut at a time.}
We show now that this factorization can be computed in quadratic time in the number of nodes, scaled by the cut size. \chan{Let $T_m(\cdot)$ be the time complexity of a parallel algorithm, assuming $m$ processors are available \cite{jaja1992parallel}. Then, we have the following result.}
\begin{theorem}
\label{th:master-eigen}
Let $\mathcal{G}\in\mathcal{F}(L, \lbrace \mathcal{G}_i\rbrace_{i = 1}^m, k)$. Assume the eigendecomposition of leaf subgraph $\mathcal{G}_i$ can be computed in $O(f_i(n))$ time. Then, we can find the Cauchy decomposition of the eigenvectors basis of the graph Laplacian $\bt L$ in $O(k n^2 + \sum_i f_i(n))$ for a sequential solver and in time $T_m(n, k) = O(k n^2 + \max_i f_i(n))$ for a parallel solver.
\end{theorem}
\begin{proof}
    See \cref{app:cauchy_eigen}.
\end{proof}
For simplicity, we focus on the parallel case. We highlight two cases: 1) when $L = \log_2n$, the base graphs become trivial and the complexity is $O(k n^2)$, and 2) when $k \ll n, k=O(1)$, i.e., connections between graphs are sparse, the complexity becomes $O(n^2 + \max_i f_i(n))$. An outline of the algorithm is depicted in \cref{alg:cauchy_sketch_eigen}, and the full algorithm is shown in \cref{alg:cauchy_full}. Next, we show how to find a suitable \chan{partition} for a given graph.

\begin{algorithm}[t]
\caption{Cauchy factorization}
\label{alg:cauchy_sketch_eigen}
\begin{algorithmic}[1]
\Require HGF graph $\mathcal{G}\in\mathcal{F}(L,\{\mathcal{G}_i\}_{i=1}^m,k)$.
\Ensure Eigenvalues $\boldsymbol{\lambda}$, Cauchy history $\mathcal{H}$.

\State \textbf{1. Initialization}
\State Compute leaf eigendecompositions $\{(\mathbf{U}_i,\boldsymbol{\lambda}_i)\}$.
\State Initialize set of transformed updates, $\mathcal{Z} = \emptyset$.
\For{each bridge edge $e=(u,v)$}
    \State $\mathcal{Z}\gets \mathcal{Z}\cup (\mathbf{e}_u - \mathbf{e}_v)$
\EndFor
\State Project all $\mathbf{z}\in\mathcal{Z}$ into the leaf bases $\mathbf{U}_i$.

\Statex
\State \textbf{2. Hierarchical merge}
\For{level $\ell = 1$ to $L$}
    \For{each pair of subgraphs to merge}
        \State Concatenate eigenvalues: $\boldsymbol{\lambda}\gets[\boldsymbol{\lambda}_L;\boldsymbol{\lambda}_R]$.
        \For{each bridge edge $e$ connecting them}
            \State Retrieve projection vector $\mathbf{z}\in\mathcal{Z}$.
            \State \textbf{Solve} the secular equation to get $\boldsymbol{\lambda}_{\mathrm{new}}$.
            \State \textbf{Construct} OCLM $\mathbf{C}$ from $(\mathbf{z},\boldsymbol{\lambda},\boldsymbol{\lambda}_{\mathrm{new}})$.
            \State \textbf{Update} relevant $\mathbf{z}'\in\mathcal{Z}$ via $\mathbf{z}'\gets\mathbf{C}^\top\mathbf{z}'$.
            \State $\boldsymbol{\lambda}\gets \boldsymbol{\lambda}_{\mathrm{new}}$, \quad $\mathcal{H}\gets \mathcal{H}\cup\{\mathbf{C}\}$
        \EndFor
    \EndFor
\EndFor

\State \Return $\boldsymbol{\lambda}$,\; $\mathcal{H}$
\end{algorithmic}
\end{algorithm}

\section{Finding a suitable HGF}
\label{sec:finding_hgf}
\chan{The complexity of our method scales with the cut size $k$, not with the edge count; sparse graphs help insofar as they admit small-cut partitions.} Although every graph admits an HGF decomposition, when $k = O(n)$ the decompositions may not be favorable for computation. We provide a greedy algorithm that partitions any graph into subgraphs directly minimizing the cost of computing the Cauchy factorization. When decomposition yields diminishing returns, we use spectral sparsification to reduce the cut size.

\subsection{Greedy HGF construction}
Based on the GFT factorization cost  (\cref{th:master-eigen}), we develop a heuristic that favors balanced partitions with few bridge edges.
For a given graph, we generate candidate partitions using balanced cuts with varying balance constraints (e.g., spectral bisection based on the Fiedler vector with varying balance tolerances) \cite{ng2001spectral}. To accept a partition, we evaluate whether the split offers computational benefits based on \cref{th:master-eigen}. Let $f(\mathcal{G})$ denote the theoretical complexity cost of processing graph $\mathcal{G}$, i.e., the eigendecomposition complexity. Given a graph $\mathcal{G}$ and a candidate split into subgraphs $\mathcal{G}_1, \mathcal{G}_2$, we accept the partition only if:
\begin{equation}
    n^2 k + \max_{i = 1, 2} \, f(\mathcal{G}_i)< f(\mathcal{G}).
\end{equation}
If this condition holds, we recur on $\mathcal{G}_1$ and $\mathcal{G}_2$. If no balanced cut satisfies this condition, we mark the subgraph as a leaf or apply the sparsification strategy below.

Regarding spectral bisection, using the Lanczos algorithm, each iteration requires a matrix-vector multiplication with the Laplacian, which
costs $O(|E|)$ time for sparse graphs. \chan{Since we have to repeat this operation on each level, }the total runtime is $O(|E|\,2^L)$, with $L = \log(n)$. Hence, computing the Fiedler vector is nearly linear in the number of edges
and is not a computational bottleneck.

\subsection{Cut sparsification}
When partitioning fails to reduce the cost function, rather than keeping all edges crossing the cut, we remove some of them via a spectral sparsification \cite{spielman2008graph} applied locally, i.e., only to subgraph cuts. 
\begin{theorem}[Interface sparsification for HGFs]
\label{thm:hgf_sparsification}
Let $\mathcal{G} \in \mathcal{F}(L, \{\mathcal{G}_i\}_{i=1}^m, k)$ and $\varepsilon \in (0,1)$.
Then there exists a graph $\mathcal{G}' \in \mathcal{F}(L, \{\mathcal{G}_i\}_{i=1}^m, k')$
such that their combinatorial Laplacians $\bt L$ and $\bt L'$ satisfy $\forall \bt x \in \mathbb{R}^n$,
\begin{equation}
(1-\varepsilon)\,\bt x^\top \bt L \bt x
\;\le\;
\bt x^\top \bt L' \bt x
\;\le\;
(1+\varepsilon)\,\bt x^\top \bt L \bt x,
\end{equation}
with $k' \leq O\!\left( \varepsilon^{-2}\right)$
that can be constructed in $O(|E| \, \log n)$.
\end{theorem}
\begin{proof}
Fix the hierarchical partition defining $\mathcal{G}$.
At each level, consider the subgraph induced by the bridge edges connecting
pairs of subgraphs.
Applying spectral sparsification \cite{spielman2008graph} independently to each
such interface yields a graph with at most $k' = O(\varepsilon^{-2})$ edges per
interface, while preserving the Laplacian quadratic form within $(1\pm\varepsilon)$.
Since sparsification preserves vertex
partition, the resulting graph $\mathcal{G}'$ belongs to the same HGF with reduced interface size.
\end{proof}
By \cref{thm:hgf_sparsification}, any graph can be replaced by a spectrally equivalent member of an HGF with bounded cut size, ensuring that the complexity bounds of \cref{th:master-eigen} apply up to a controllable approximation error. From the perspective of Cauchy factorization, sparsification can be viewed as a theoretically grounded form of truncation, limiting the effective number of the Cauchy updates while preserving the Laplacian quadratic form up to a controlled error. For spectral filters $g(\cdot)$ that are Lipschitz continuous on the spectrum of $ \bt L$, the output of a layer based on $\bt L'$ deviates from that based on $\bt L$ by at most
$O(\varepsilon)$, with the constant depending on the Lipschitz constant of $g(\cdot)$ \cite{levie2021transferability}.

\section{L2G-Net}
\label{sec:cauchy_gnn}

\chan{Given the Cauchy factorization, we can add learnable spectral filters in between Cauchy factors while preserving the overall complexity and providing more flexibility than GFT-based spectral GNNs. This is the idea underpinning L2G-Net (\cref{fig:graph_multires}b).} L2G-Net can be used as a drop-in replacement for standard spectral GNNs, potentially involving pointwise non-linearities and residual connections.

Assume a stack of $M$ spectral layers, acting over $\bt X_0\in\mathbb{R}^{n\times n_{\sf f}}$, where $n_{\sf f}$ is the number of channels. Let $\bt Z_m = \bt X_m\bt W_m$. Then, for $c=1, \hdots, n_{\sf f}$,
\begin{equation}
    \bt X_{m+1, c} = \bt Ug_{\theta, c}(\pmb \lambda)\bt U^\top \bt Z_{m, c}, \quad m = 0, \hdots, M-1.
\end{equation}
At a high level, our construction modifies the forward transform $\bt U^\top \bt Z_{m, c}$ by applying spectral processing independently on small base subgraphs using learnable local graph filters. Then,  information is progressively mixed across larger graph components through a sequence of Cauchy factors (cf.~\cref{thm:cauchy_factorization}). Finally, the synthesis transform is computed, again using the same factorization.

We start from the Cauchy factorization in \cref{thm:cauchy_factorization}. We insert learnable local filters before and after multiplying by the set of Cauchy factors corresponding to the cuts between each pair of subgraphs. Let these filters be $g_{r, p}(\cdot)$ for the $r$th level, with $r = 0, \hdots, L$, and considering the $p$th pair of subgraphs, with $p = 1, \hdots, 2^r-1$. Let $\bt D_{r, p}$ be the product of the Cauchy factors corresponding to the same cut. Given the filtered spectral representations at level $r-1$ on each subgraph, $\bt H_{p, 1}^{r-1}$ and $\bt H_{p, 2}^{r-1}$, we compute:
\begin{equation}
    \label{eq:hierarchical_mix}
        \bt H^{(r)}_{p}
    = 
    g_{r,p}(\pmb \lambda_{r,p})
    \bt D_{r,p}
    \begin{bmatrix}
        \bt H^{(r-1)}_{p_{l}} \\
        \bt H^{(r-1)}_{p_{r}}
    \end{bmatrix},
\end{equation}
where $\pmb \lambda_{r,p}$ are the eigenvalues of the merged subgraph at level $r$. By repeating this procedure recursively up to $r=L$, we obtain the spectral representation $\bt H^{(L)}$. We will denote by $\bt U(\bt \Phi)$ the matrix replacing $\bt U$ with the concatenation of GFTs, filtering, and Cauchy matrices, where $\bt \Phi$ represents learnable parameters for all $r$ and $p$. The output in the node domain is given by $\bt X_{\mathrm{out}} = \bt U g_{\theta}(\pmb \lambda) \bt U(\bt \Phi) \bt X_m$. We remark that products by $\bt U$ are computed using the factorization, so $\bt U$ is never explicitly computed. 

\chan{Remarkably, the partition and Cauchy factorization depend only on $\mathcal{G}$; learnable parameters appear only in the filters $g_{r,p}(\cdot)$, so the architecture is fixed by the graph before any training begins. The following result characterizes the expressiveness of L2G-Net.}
\chan{
\begin{theorem}[Expressiveness]\label{thm:expressiveness}
Let $\mathcal{G} \in \mathcal{F}(L, \{\mathcal{G}_i\}_{i=1}^m, k)$ with $k \geq 1$. The class of operators realizable by L2G-Net strictly contains the class of global spectral filters $g(\mathbf{L})$.
\end{theorem}}
\begin{proof}[Proof]
See \cref{app:expressiveness}.
\end{proof}
\chan{A spectral filter $g(\mathbf{L})$ applies the same response across the entire graph. L2G-Net applies different spectral responses to different subgraphs before merging via Cauchy factors, producing operators that cannot be represented as $g(\bt L)$. This degree of freedom is given by the partition structure.}

\section{Experiments}
First, we validate our theoretical complexity predictions on synthetic graphs with controlled structure. 
Second, we evaluate our method on real-world graphs that may deviate from the ideal hierarchical setup. While our primary contribution is to provide a principled framework for exploiting hierarchical graph structure when it is present, the GNN experiments demonstrate that the approach degrades gracefully under approximation. All experiments are conducted on an Intel(R) Core(TM) i9-9900K CPU @ 3.60GHz with an NVIDIA GeForce RTX 2070 GPU.

\begin{table}[t]
\centering
\caption{Average test accuracy with standard deviation for $10$ random train/test partitions on the benchmark in \cite{platonov2023critical}. We report the best-performing methods per family. We show average Acc for \textit{Roman-empire} (R. Emp.) and \textit{Amazon-ratings} (Am. Rat.) and AUC for \textit{Tolokers} (Tolok.) and \textit{Minesweeper} (Mines.). Higher is better ($\uparrow$). See also \cref{tab:table9}.}
\label{tab:main_results}
\small
\setlength{\tabcolsep}{5.5pt}
\renewcommand{\arraystretch}{1.0}
\begin{tabular}{lcccc}
\toprule
\textbf{Method} & \textbf{R. Emp.} & \textbf{Am. Rat.} & \textbf{Mines.} & \textbf{Tolok.} \\
\midrule
GCN & 73.69{\scriptsize (0.7)} & 48.70{\scriptsize (0.6)} & 89.75{\scriptsize (0.5)} & 83.64{\scriptsize (0.7)} \\

\midrule

CO-GNN
& 91.57{\scriptsize (0.3)}
& \underline{54.17}{\scriptsize (0.4)}
& 97.31{\scriptsize (0.4)}
& 84.45{\scriptsize (1.2)} \\

Polynormer
& \textbf{92.55}{\scriptsize (0.3)}
& \textbf{54.81}{\scriptsize (0.5)}
& 97.46{\scriptsize (0.4)}
& \textbf{85.91}{\scriptsize (0.7)} \\

MP-SSM
& 90.91{\scriptsize (0.5)}
& 53.65{\scriptsize (0.7)}
& 95.33{\scriptsize (0.7)}
& 85.26{\scriptsize (0.9)} \\

St-ChebNet
& 92.03{\scriptsize (0.9)}
& 53.15{\scriptsize (0.2)}
& 95.71{\scriptsize (2.3)}
& 85.55{\scriptsize (3.4)} \\

\midrule

\textbf{Ours}
& \underline{92.12}{\scriptsize(1.1)}
& 53.39\scriptsize (0.6)
& \textbf{97.50}{\scriptsize (0.3)}
& \underline{85.57}{\scriptsize (0.6)} \\
\bottomrule
\end{tabular}
\end{table}

\begin{table}[t]
\centering
\caption{Runtime (s) for computing the eigenvalue decomposition (ED) and the Cauchy factorization (CF) on heterophilous graphs from \cite{platonov2023critical}. }
\label{tab:runtime_results}
\small
\setlength{\tabcolsep}{8pt}
\renewcommand{\arraystretch}{1.05}
\begin{tabular}{lcccc}
\toprule
\textbf{Method} & \textbf{R. Emp.} & \textbf{Am. Rat.} & \textbf{Mines.} & \textbf{Tolok.} \\
\midrule
ED
& 731.12
& 896.56
& 92.34
& 117.57 \\

CF (ours)
& \textbf{232.49}
& \textbf{281.30}
& \textbf{32.21}
& \textbf{91.33} \\
\bottomrule
\end{tabular}
\end{table}

\subsection{Synthetic experiments}

We validate the predictions of \cref{th:master-eigen} measuring the runtime of the proposed Cauchy factorization on synthetic graphs. We isolate the dependence on graph size and cut size. All experiments are run on CPU.

\paragraph{Setup.}
We generate random Barabási-Albert graphs \cite{barabasi1999emergence}. For each graph with $n$ nodes, we first compute a balanced partition using spectral bisection based on the Fiedler vector of the combinatorial Laplacian. We repeat this on each subgraph to obtain a partition with $4$ subgraphs. To control the cut size, we sparsify the cut edges, retaining a fixed target number of bridge edges.

\paragraph{Runtime vs graph size (\cref{fig:runtime_nodes}).}
We compare
(i) full eigendecomposition of the Laplacian via \texttt{numpy.linalg.eigh} and 
(ii) our Cauchy factorization. 
%
We report the median runtime across three runs vs number of nodes $n$ for a fixed cut size $k=5$. Eigendecomposition scales with $O(n^3)$ and becomes impractical beyond $n\approx 10^4$. 
The Cauchy factorization follows a clear quadratic trend, consistent with the $\mathcal{O}(kn^2)$ complexity predicted by our analysis. As $n$ increases, the cost of the base eigendecompositions of the individual subgraphs becomes non-negligible, causing the overall trend to gradually approach cubic behavior. 
This effect is expected and can be mitigated by increasing the number of partitions; when the number of levels is $L = \log n$, the complexity of the base eigendecompositions is trivial. The preprocessing overhead (spectral cut and sparsification) remains negligible across all graph sizes considered.

\begin{figure}
    \centering
    \includegraphics[width=0.95\linewidth]{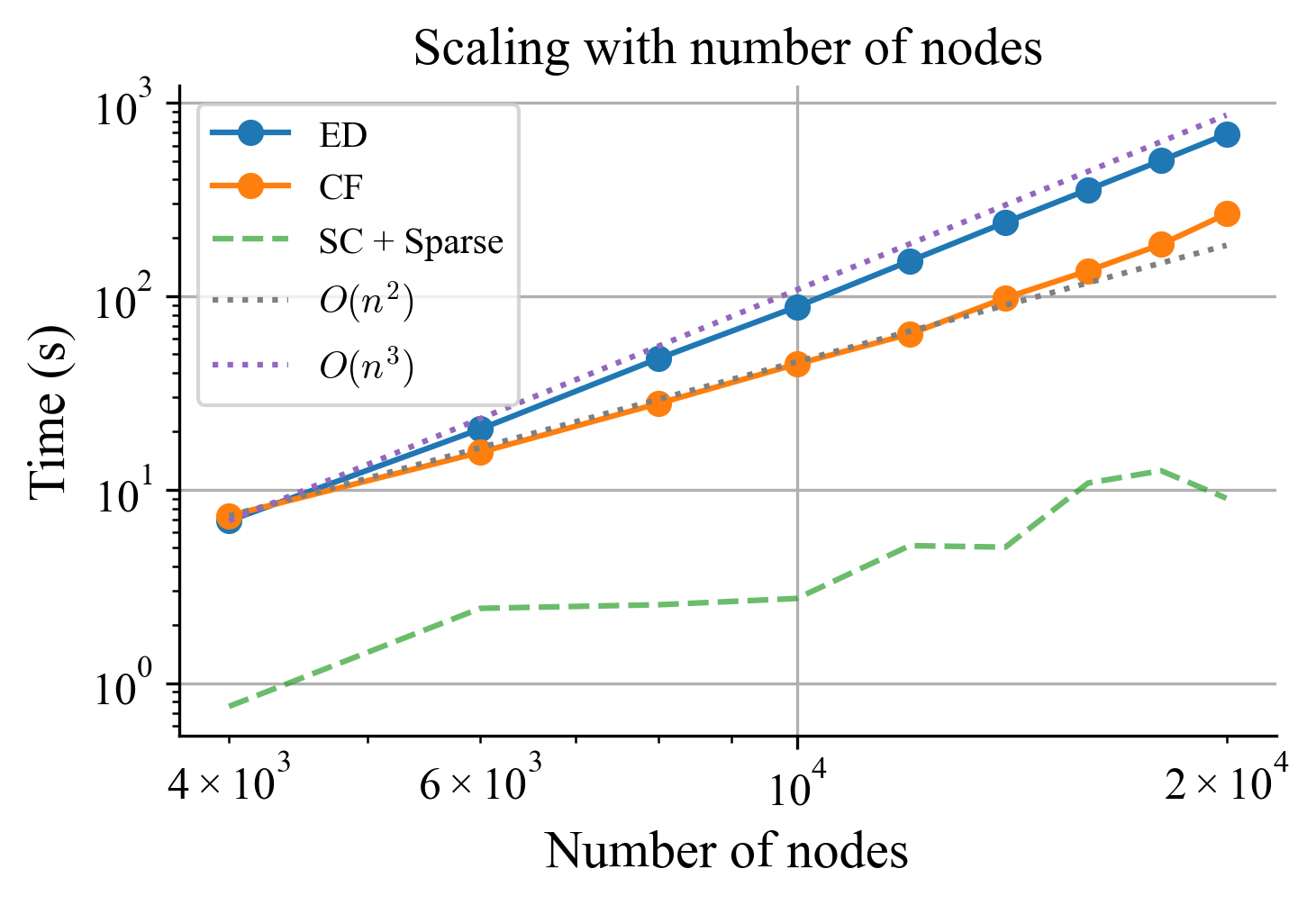}
    \caption{Runtime on a random graph for eigendecomposition (ED), Cauchy factorization (CF), and preprocessing (SC+Sparse). The Cauchy factorization scales quadratically with the number of nodes. The preprocessing time is negligible.}
    \label{fig:runtime_nodes}
\end{figure}

\textbf{Runtime vs cut size (\cref{fig:scalability_k}).}
For a fixed graph size ($n=8000$), the runtime of the proposed Cauchy factorization grows linearly with the number of crossing edges, $k$, which corresponds to the number of the Cauchy updates. This confirms that the computational cost is governed by the cut size rather than by the total graph size. 
The preprocessing step based on spectral cuts and sparsification exhibits an approximately constant runtime, as promised by \cref{thm:hgf_sparsification}. 
These results show that our method remains efficient because you can sparsify first to reduce the cut size and this sparsification has limited complexity.

\begin{figure}
    \centering
    \includegraphics[width=0.95\linewidth]{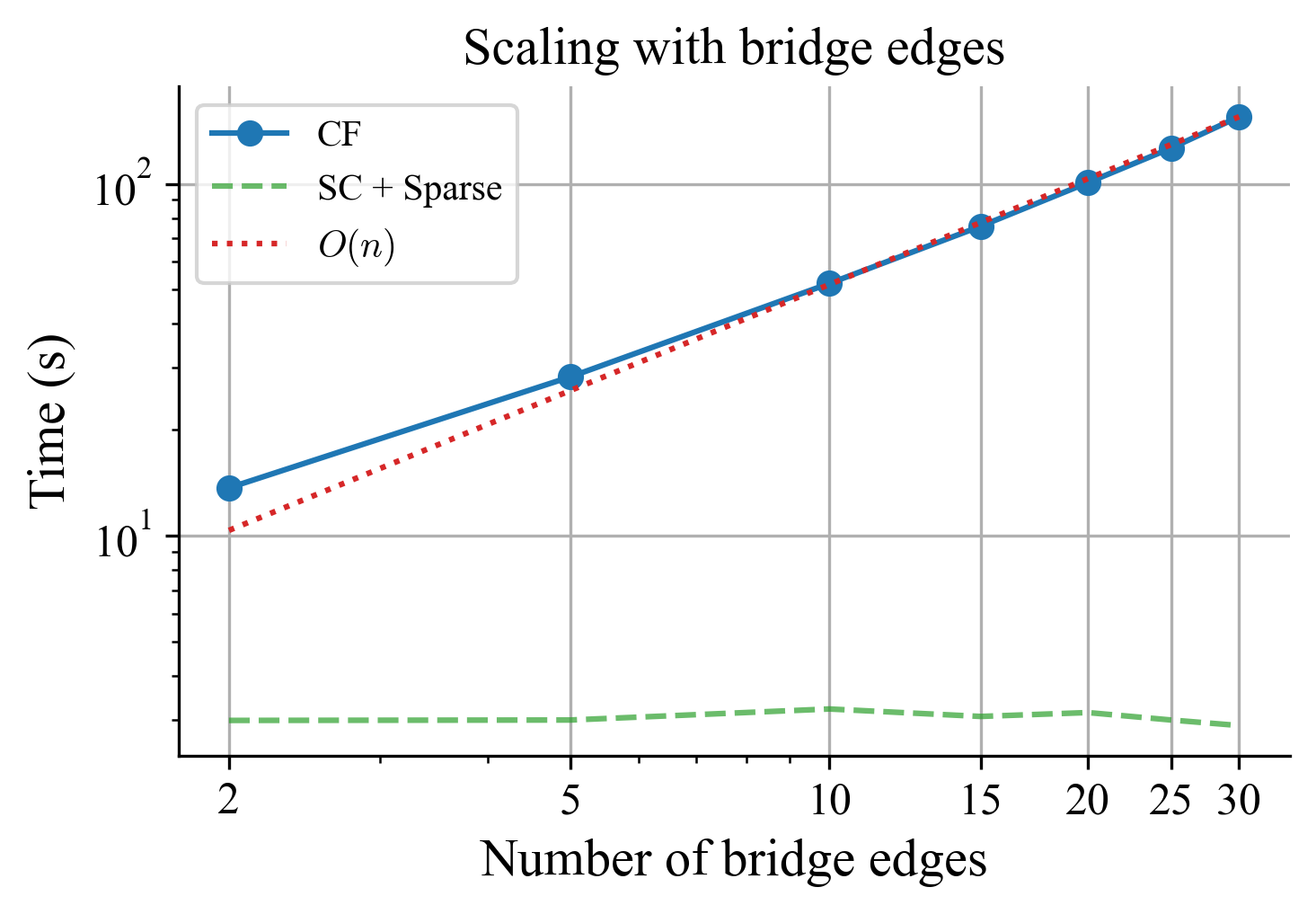}
    \caption{Runtime against $k$ on a random graph for Cauchy factorization (CF) and preprocessing (SC+Sparse). The complexity of the Cauchy factorization scales linearly with the cut size. The preprocessing time is small and roughly constant with $k$.}
    \label{fig:scalability_k}
\end{figure}

\subsection{Transductive heterophilous benchmarks}

\paragraph{Setup.}We evaluate the proposed method on a heterophilous benchmark suite  \cite{platonov2023critical}, which comprises large-scale graphs designed to challenge MPNNs. In particular, we focus on \textit{Roman-empire}, \textit{Amazon-ratings}, \textit{Minesweeper}, and \textit{Tolokers}. These datasets combine non-local dependencies with graph sizes for which the global GFT remains feasible but is very slow, making them well-suited for assessing our approach. We rely on the normalized Laplacian, although we sparsify based on the combinatorial Laplacian \cite{batson2014twice}. We use splines to parametrize our filters. Our method first partitions the graph into two subgraphs (one level of hierarchy) and sparsifies the cut. A single shared filter is used across all feature channels and layers (cf.~ \cref{app:arch_details}).

\paragraph{Performance (\cref{tab:main_results}).} L2G-Net achieves competitive or superior performance across all benchmarks, exceeding the best baselines \cite{deng2024polynormer} on \textit{Minesweeper} and remaining close to the strongest methods on the others. Notably, while attention-based models such as Polynormer attain strong accuracy, they do so at the cost of learned graph structures that are harder to interpret (cf.~\cref{fig:gradcam}). 

\paragraph{Runtime (\cref{tab:runtime_results}).} We report the cost of CF and the eigendecomposition. CF yields speedups across all datasets. \textit{Tolokers} is a worst-case scenario: this graph is particularly dense, so the number of edges between subgraphs remains high after sparsification. Even in this case, we obtain runtime gains over direct eigendecomposition.

\paragraph{Learnable parameters (\cref{fig:performance_vs_parameters}).} We compare Global GFT, L2G-Net, and Polynormer across the four datasets. The topological inductive bias allows L2G-Net to compete with state-of-the-art methods while using orders of magnitude fewer learnable parameters. Moreover, L2G-Net consistently outperforms the Global GFT baseline across all datasets, highlighting the benefit of localized spectral processing over fully global spectral representations.

\begin{figure}
    \centering
    \includegraphics[width=\linewidth]{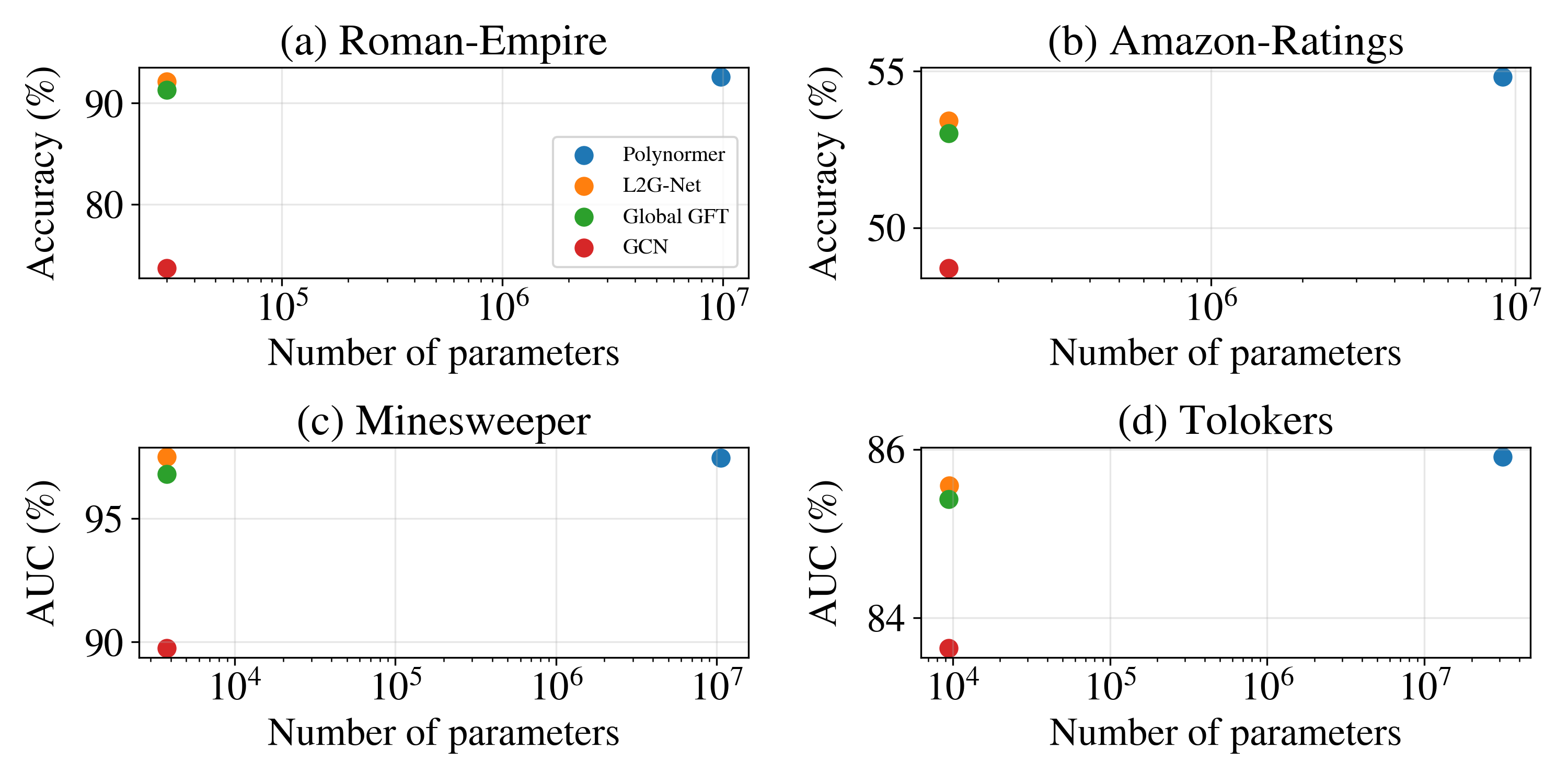}
    \caption{Performance vs complexity comparison. L2G-Net achieves competitive or superior performance with orders of magnitude fewer learnable parameters than state-of-the-art baselines, and outperforms the Global GFT across all datasets.}
    \label{fig:performance_vs_parameters}
\end{figure}

\subsection{Long-range inductive benchmarks}
\chan{To demonstrate the effectiveness of our method on inductive tasks, we evaluate L2G-Net on the long-range molecular benchmarks \emph{peptides-func} and \emph{peptides-struct} from \cite{dwivedi2022long}, which consist of inductive graph-level prediction tasks on small molecular graphs (\cref{tab:lrgbench}). These datasets allow for full Laplacian eigendecomposition; hence, we can isolate architectural inductive bias without confounding results with sparsification effects. Results show that L2G-Net performs within one standard deviation of state-of-the-art long-range models, including spatio-spectral hybrids such as \cite{geisler2024spatio}, without relying on positional encodings or graph rewiring. Moreover, L2G-Net outperforms spectral methods such as Specformer, ChebNet, Stable ChebNet, and ChebNetII. This indicates that the hierarchical spectral factorization underlying L2G-Net captures long-range dependencies in inductive settings without additional structural augmentation. Furthermore, our method can also be extended with positional encodings (PEs): when considering Laplacian and Random Walk PEs \cite{dwivedi2022graph}, we can obtain better accuracy in both benchmarks, demonstrating that L2G-Net and PEs are complementary.}

\begin{table}[t]
\centering
\caption{\chan{Time (minutes) to compute the eigendecomposition (ED) versus the CF factorization on City-Networks. ED runs out of memory in all cases; EDP extrapolates the $O(n^3)$ trend}.}
\label{tab:scaling}
\setlength{\tabcolsep}{3.5pt}
\begin{tabular}{lcccc}
\toprule
Method & \textbf{Paris} & \textbf{Shanghai} & \textbf{LA} & \textbf{London} \\
\midrule
ED              & OOM     & OOM      & OOM      & OOM \\
EDP  & 50412.82 & 121932.43 & 131543.54 & 1632123.32 \\
CF   & \textbf{17.91} & \textbf{45.61} & \textbf{61.53} & \textbf{144.22} \\
\bottomrule
\end{tabular}
\end{table}

\subsection{Long-range transductive benchmarks} \chan{We consider the dataset from \cite{liang2025towards}, which contains graphs with more than $10^{5}$ nodes. We use different levels of decomposition ($L=4$ for Paris and Shanghai, $L=5$ for LA, and $L=6$ for London). We sparsify keeping only $2$ edges between each pair of subgraphs. Table~\ref{tab:scaling} reports the time required to compute the factorization. Full eigendecomposition (ED) is infeasible on these graphs due to out-of-memory (OOM) errors; we include a projected ED runtime obtained by extrapolating the cubic scaling trend from smaller graphs, to give a sense of the asymptotic gap. The Cauchy factorization is three-four orders of magnitude faster than projected ED. The largest graph in the suite (London, 569k nodes) is factorized in roughly 144 minutes, which makes exact spectral processing tractable at scales where it was previously out of reach.}

\chan{We also show accuracy results (\cref{tab:citynetwork}). L2G-Net is competitive with state-of-the-art methods, and outperforms spectral baselines such as ChebNet. These results show that 1) L2G-Net scales to graphs where computing the GFT basis is unfeasible, and 2) the local-to-global inductive bias yields gains in long-range setups.}

\subsection{Local processing}
\label{sec:gradcam}
\paragraph{Setup.} To analyze the localization of predictive importance across graph nodes, we perform a node-level attribution study using Grad-CAM \cite{selvaraju2017grad} on the \textit{Minesweeper} dataset. For each evaluation sample, we compute Grad-CAM scores at the node level, normalize them to sum to one, and sort nodes by decreasing importance. We then compute the cumulative sum of importance as a function of the fraction of nodes retained, yielding a cumulative contribution curve per sample. These curves are finally averaged across the evaluation set. \chan{We include results for other graphs in \cref{app:gradcam}.}

\paragraph{Results (\cref{fig:gradcam}).} L2G-Net concentrates its predictive contributions on a small subset of nodes, reflecting its local bias. The Global GFT baseline spreads importance over a larger portion of the graph, consistent with global spectral representations. For Polynormer \cite{deng2024polynormer}, graph learning combined with attention-based mechanisms leads to more diffuse contributions, making explanations difficult to interpret with respect to the input geometry. \chan{We also show the filters learned by L2G-Net in \cref{app:learned_filters}.}

\begin{figure}
    \centering
    \includegraphics[width=0.95\linewidth]{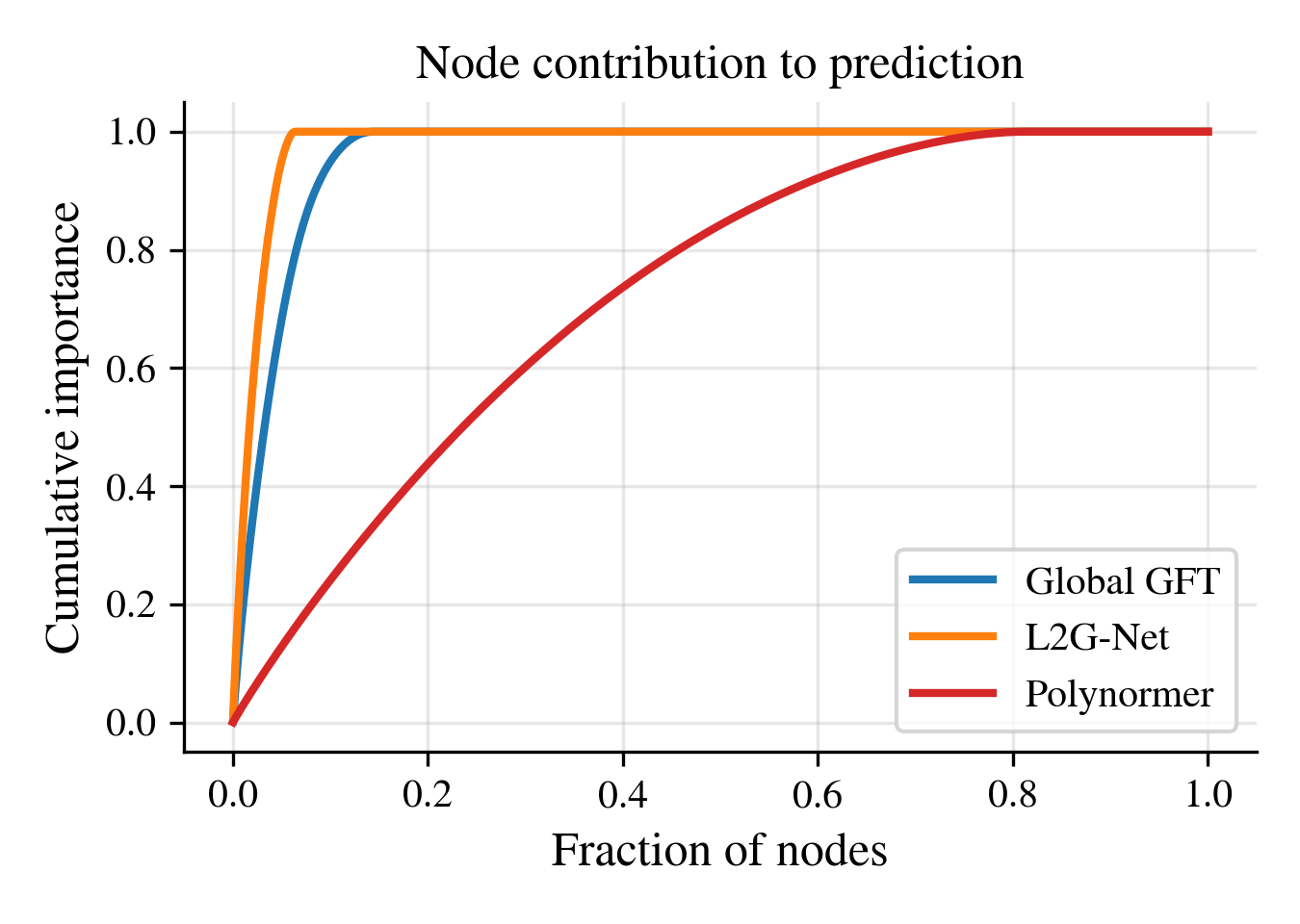}
    \caption{Cumulative node contribution to the prediction on \textit{Minesweeper}. Average across validation set; shaded areas indicate confidence intervals. L2G-Net localizes predictive power on a smaller fraction of nodes than Global GFT. The node contributions of Polynormer are less localized in the original graph.}
    \label{fig:gradcam}

    \vspace{-1em}
\end{figure}

\begin{table*}[t]
\centering
\caption{\chan{Memory consumption (GB) of full GFT vs.\ our Cauchy factorization (CF), using 64-bit precision. On \cite{platonov2023critical} we use $L=1$; on City-Networks \cite{liang2025towards}, where full GFT is infeasible at practical hardware limits, we use deeper hierarchies.}}
\label{tab:memory}
\setlength{\tabcolsep}{11pt}
\begin{tabular}{lcccc|cccc}
\toprule
Method & \textbf{Mines.} & \textbf{Tolok.} & \textbf{Am. Rat.} & \textbf{R. Emp.} & \textbf{Paris} & \textbf{Shanghai} & \textbf{LA} & \textbf{London} \\
\midrule
Full GFT  & 0.80 & 1.11 & 4.80 & 4.11 & 103.97 & 270.61 & 463.06 & 2588.22 \\
CF (ours) & \textbf{0.20} & \textbf{0.28} & \textbf{1.21} & \textbf{1.03} & \textbf{12.21} & \textbf{16.98} & \textbf{14.58} & \textbf{40.96} \\
\bottomrule
\end{tabular}

\vspace{-1em}
\end{table*}

\begin{table}[t]
\centering
\small
\setlength{\tabcolsep}{5pt}
\caption{\chan{LRGB results \cite{dwivedi2022long}; top performers per category shown, full results in Appendix~\ref{app:more_exper}. AP on peptides-func (higher is better), MAE on peptides-struct (lower is better). Trans.~= transformer, Rewir.~= rewiring, SS = state space.}}
\begin{tabular}{llcc}
\toprule
\textbf{Type} & \textbf{Model} & \textbf{pep.-func} ($\uparrow$) & \textbf{pep.-struct} ($\downarrow$) \\
\midrule
\multirow{2}{*}{Trans.}
& Graph ViT     & 69.42 $\pm$ 0.75 & 0.2449 $\pm$ 0.0016 \\
& GRIT          & 69.88 $\pm$ 0.82 & 0.2460 $\pm$ 0.0012 \\
\midrule
Rewir.
& DRew-GCN+PE   & 71.50 $\pm$ 0.44 & 0.2536 $\pm$ 0.0015 \\
\midrule
\multirow{2}{*}{SS}
& GMN           & 70.71 $\pm$ 0.83 & 0.2473 $\pm$ 0.0025 \\
& MP-SSM        & 69.93 $\pm$ 0.52 & 0.2458 $\pm$ 0.0017 \\
\midrule
\multirow{5}{*}{GNN}
& GCN           & 68.60 $\pm$ 0.50 & 0.2460 $\pm$ 0.0007 \\
& GRAMA         & 70.93 $\pm$ 0.78 & \textbf{0.2436} $\pm$ 0.0022 \\
& PH-DGN        & 70.12 $\pm$ 0.45 & 0.2465 $\pm$ 0.0020 \\
& S2GCN         & \textbf{72.75} $\pm$ 0.66 & \underline{0.2467} $\pm$ 0.0019 \\
& \hspace{1em}+PE & \textbf{73.11} $\pm$ 0.66 & \underline{0.2447} $\pm$ 0.0032 \\
\midrule
\multirow{2}{*}{Ours}
& L2G-Net       & \underline{72.14} $\pm$ 0.24 & 0.2479 $\pm$ 0.0012 \\
& \hspace{1em}+PE & \underline{72.46} $\pm$ 0.25 & 0.2462 $\pm$ 0.0011 \\
\bottomrule
\end{tabular}
\label{tab:lrgbench}
\end{table}

\setlength{\tabcolsep}{5pt}

\begin{table}[t]
\centering
\caption{\chan{Average test accuracy with standard deviation for $4$ random seeds on City-Networks \cite{liang2025towards}. Higher is better. We show selected methods, full results in Appendix~\ref{app:more_exper}.}}
\label{tab:citynetwork}
\begin{tabular}{lcccc}
\toprule
Model & \textbf{Paris} & \textbf{Shanghai} & \textbf{LA} & \textbf{London} \\
\midrule
ChebNet   & 54.1$_{(0.2)}$ & 66.5$_{(0.1)}$ & 61.4$_{(0.4)}$ & 54.7$_{(0.2)}$ \\
GCN       & 53.2$_{(0.3)}$ & 62.1$_{(0.2)}$ & 58.3$_{(0.3)}$ & 50.1$_{(0.7)}$ \\
SAGE & 54.6$_{(0.2)}$ & 68.3$_{(0.5)}$ & 61.4$_{(0.3)}$ & 55.4$_{(0.2)}$ \\
\midrule
GraphGPS  & 52.1$_{(0.6)}$ & 63.0$_{(0.5)}$ & 59.8$_{(0.5)}$ & OOM \\
Exphormer & 55.1$_{(0.8)}$ & 70.2$_{(0.4)}$ & 63.8$_{(0.6)}$ & 49.5$_{(0.4)}$ \\
SGFormer  & 52.0$_{(0.8)}$ & 64.1$_{(0.3)}$ & 60.1$_{(0.7)}$ & 48.3$_{(0.3)}$ \\
\midrule
\textbf{Ours} & 55.4$_{(0.4)}$ & 69.8$_{(0.5)}$ & 63.5$_{(0.7)}$ & 55.6$_{(0.6)}$\\
\bottomrule
\end{tabular}

\vspace{-1em}
\end{table}
\setlength{\tabcolsep}{6pt}

\chan{
\subsection{Memory consumption}
Storing a Cauchy matrix has complexity $O(n)$. Table~\ref{tab:memory} compares the memory required by the full GFT against our Cauchy factorization. On \cite{platonov2023critical}, where both methods are tractable, the Cauchy factorization reduces memory by roughly $4\times$ across all four datasets at the shallowest setting ($L=1$). The gap becomes more clear on City-Networks: storing the full GFT is well beyond standard hardware. Our factorization brings these costs down substantially, making exact spectral processing feasible.
}
\section{Conclusion}
We introduced an exact factorization of the GFT for arbitrary graphs based on Cauchy matrices and graph partitions. This factorization enables a divide-and-conquer strategy for spectral processing, reducing the cost of exact GFT-based operators from ${O}(n^3)$ to ${O}(n^2)$, scaled by the subgraph cut size. We proposed an algorithm to find partitions that minimize the factorization cost and, when this remains costly, reduce complexity via cut sparsification. Building on these results, we introduced L2G-Net, a new class of spectral GNNs that combine local and global spectral processing by construction. \chan{L2G-Net is most appropriate in regimes where spectral processing is desired (e.g., long-range dependencies) and polynomial filters would require high order or deep architectures, leading to oversquashing or numerical instability.} Experiments on synthetic and real-world benchmarks validate the theoretical analysis, confirming the predicted runtime scaling and showing that the resulting models remain competitive while requiring  fewer parameters than transformer-based methods and lower preprocessing cost than the global GFT. \chan{In strictly large-scale settings where linear complexity is key, polynomial filters and message passing remain preferable.}

The factorization extends beyond the static undirected setting. For dynamic undirected graphs, edge insertions and deletions can be handled by adding a single Cauchy factor, while node additions require recomputing only the affected subgraph and its cut factors. Signed graphs are also supported since our result holds for any symmetric matrix; directed graphs (e.g., via Hermitian Laplacians) are left for future work. L2G-Net opens new directions for fast and structured spectral algorithms on graphs.

\newpage

\section*{Acknowledgements}
The authors would like to thank Jhony Giraldo for feedback on the initial versions of this manuscript. SFM would like to thank Narjes Nourzad for help designing Figure 1.


\section*{Impact Statement}
This paper presents work aimed at advancing the field of machine learning. There are many potential societal consequences of our work, none of which we feel must be specifically highlighted here.

\bibliographystyle{plainnat} 
\bibliography{refs}       

@string{ IEEE-Proc       = "Proc. IEEE"}

@string{ IEEE-SP-Mag     = "IEEE Signal Process. Mag."}

@string{ neurips       = "Proc. Advan. Neural Inf. Process. Syst."}

@string{ icassp     = "Proc. IEEE Int. Conf. Acoust., Speech, and Signal Process."}

@article{bunch1978rank,
  author  = {Bunch, James R. and Nielsen, Christopher P. and Sorensen, Danny C.},
  title   = {Rank‑One Modification of the Symmetric Eigenproblem},
  journal = {Numerische Mathematik},
  volume  = {31},
  pages   = {31--48},
  year    = {1978},
}

@inproceedings{fernandez-menduina2025fast,
  title={Fast {DCT}+: A Family of Fast Transforms Based on Rank-One Updates of the Path Graph},
  author={Fern{\'a}ndez-Mendui{\~n}a, Samuel  and Pavez, Eduardo and Ortega, Antonio},
  booktitle=icassp,
  year={2025},
  organization={IEEE}
}

@article{fasino2023orthogonal,
  title={Orthogonal {Cauchy}-like matrices},
  author={Fasino, Dario},
  journal={Numerical Algorithms},
  volume={92},
  number={1},
  pages={619--637},
  year={2023},
  publisher={Springer}
}

@article{gastinel1960inversion,
  title={Inversion d’une matrice generalisant la matrice de {Hilbert}},
  author={Gastinel, N},
  journal={Chiffres},
  volume={3},
  pages={149--152},
  year={1960}
}

@article{ortega2018graph,
  title={Graph signal processing: Overview, challenges, and applications},
  author={Ortega, Antonio and Frossard, Pascal and Kova{\v{c}}evi{\'c}, Jelena and Moura, Jos{\'e} MF and Vandergheynst, Pierre},
  journal=IEEE-Proc,
  volume={106},
  number={5},
  pages={808--828},
  year={2018},
  publisher={IEEE}
}

@article{golub1973some,
  title={Some modified matrix eigenvalue problems},
  author={Golub, Gene H},
  journal={SIAM Rev.},
  volume={15},
  number={2},
  pages={318--334},
  year={1973},
  publisher={SIAM}
}

@article{batson2014twice,
  title={Twice-{Ramanujan} sparsifiers},
  author={Batson, Joshua and Spielman, Daniel A and Srivastava, Nikhil},
  journal={SIAM Rev.},
  volume={56},
  number={2},
  pages={315--334},
  year={2014},
  publisher={SIAM}
}

@article{clauset2008hierarchical,
  title={Hierarchical structure and the prediction of missing links in networks},
  author={Clauset, Aaron and Moore, Cristopher and Newman, Mark EJ},
  journal={Nature},
  volume={453},
  number={7191},
  pages={98--101},
  year={2008},
  publisher={Nature Publishing Group UK London}
}

@article{barthelemy2011spatial,
  title={Spatial networks},
  author={Barth{\'e}lemy, Marc},
  journal={Physics reports},
  volume={499},
  number={1-3},
  pages={1--101},
  year={2011},
  publisher={Elsevier}
}

@ARTICLE{Bronstein2017geometric,
  author={Bronstein, Michael M. and Bruna, Joan and LeCun, Yann and Szlam, Arthur and Vandergheynst, Pierre},
  journal=IEEE-SP-Mag, 
  title={Geometric Deep Learning: Going beyond {Euclidean} data}, 
  year={2017},
  volume={34},
  number={4},
  pages={18-42},
  doi={10.1109/MSP.2017.2693418}}

@article{hariri2025return,
  title={Return of {ChebNet}: Understanding and improving an overlooked {GNN} on long range tasks},
  author={Hariri, Ali and Arroyo, Alvaro and Gravina, Alessio and Eliasof, Moshe and Sch{\"o}nlieb, Carola-Bibiane and Bacciu, Davide and Dong, Xiaowen and Azizzadenesheli, Kamyar and Vandergheynst, Pierre},
  journal={Advances in Neural Information Processing Systems},
  volume={38},
  pages={136166--136196},
  year={2026}
}

@article{cai2018,
author = {Cai, Difeng and Chow, Edmond and Erlandson, Lucas and Saad, Yousef and Xi, Yuanzhe},
title = {{SMASH}: Structured matrix approximation by separation and hierarchy},
journal = {Numerical Linear Algebra with Applications},
volume = {25},
number = {6},
year = {2018}
}

@book{chung1997spectral,
  title={Spectral graph theory},
  author={Chung, Fan RK},
  volume={92},
  year={1997},
  publisher={American Mathematical Soc.}
}

@inproceedings{spielman2008graph,
  title={Graph sparsification by effective resistances},
  author={Spielman, Daniel A and Srivastava, Nikhil},
  booktitle={Proceedings of ACM symposium on Theory of computing},
  pages={563--568},
  year={2008}
}

@inproceedings{bruna2014spectral,
  title={Spectral networks and locally connected networks on graphs},
  author={Bruna, Joan and Zaremba, Wojciech and Szlam, Arthur and Lecun, Yann},
  booktitle={International Conference on Learning Representations},
  year={2014}
}

@inproceedings{ng2001spectral,
 author = {Ng, Andrew and Jordan, Michael and Weiss, Yair},
 booktitle = neurips,
 editor = {T. Dietterich and S. Becker and Z. Ghahramani},
 pages = {},
 publisher = {MIT Press},
 title = {On Spectral Clustering: Analysis and an algorithm},
 volume = {14},
 year = {2001}
}

@inproceedings{kipf2017semi,
  title={Semi-Supervised Classification with Graph Convolutional Networks},
  author={Kipf, Thomas N and Welling, Max},
  booktitle={International Conference on Learning Representations},
  year={2017}
}

@ARTICLE{magoarou2018approx,
  author={Le Magoarou, Luc and Gribonval, Rémi and Tremblay, Nicolas},
  journal={IEEE Trans on Sign. and Inform. Process. over Netw.}, 
  title={Approximate Fast Graph {Fourier} Transforms via Multilayer Sparse Approximations}, 
  year={2018},
  volume={4},
  number={2},
  pages={407-420},
  }

@inproceedings{frerix2019approximating,
  title={Approximating orthogonal matrices with effective {Givens} factorization},
  author={Frerix, Thomas and Bruna, Joan},
  booktitle={Intl. Conf. on Mach. Learn.},
  pages={1993--2001},
  year={2019},
  organization={PMLR}
}

@inproceedings{alon2020bottleneck,
  title={ON THE BOTTLENECK OF GRAPH NEURAL NETWORKS AND ITS PRACTICAL IMPLICATIONS},
  author={Alon, Uri and Yahav, Eran},
  booktitle={International Conference on Learning Representations},
  year={2021}
}

@article{rusch2023survey,
  title={A Survey on Oversmoothing in Graph Neural Networks},
  author={Rusch, T Konstantin and Bronstein, Michael M and Mishra, Siddhartha},
  journal={SAM Research Report},
  volume={2023},
  year={2023},
  publisher={ETH Zurich}
}

@article{defferrard2016convolutional,
  title={Convolutional neural networks on graphs with fast localized spectral filtering},
  author={Defferrard, Micha{\"e}l and Bresson, Xavier and Vandergheynst, Pierre},
  journal={Advances in neural information processing systems},
  volume={29},
  year={2016}
}

@inproceedings{deng2024polynormer,
  title={Polynormer: Polynomial-Expressive Graph Transformer in Linear Time},
  author={Deng, Chenhui and Yue, Zichao and Zhang, Zhiru},
  booktitle={International Conference on Learning Representations},
  year={2024},
}

@inproceedings{barbero2023locality,
  title={Locality-Aware Graph Rewiring in {GNNs}},
  author={Barbero, Federico and Velingker, Ameya and Saberi, Amin and Bronstein, Michael M and Di Giovanni, Francesco},
  booktitle={International Conference on Learning Representations},
  year={2023}
}

@article{dwivedi2020generalization,
  title={A Generalization of Transformer Networks to Graphs},
  author={Dwivedi, Vijay Prakash and Bresson, Xavier},
  journal={AAAI Workshop on Deep Learning on Graphs: Methods and Applications},
  year={2021}
}

@inproceedings{huang2023stability,
  title={On the stability of expressive positional encodings for graphs},
  author={Huang, Yinan and Lu, William and Robinson, Joshua and Yang, Yu and Zhang, Muhan and Jegelka, Stefanie and Li, Pan},
  booktitle={International Conference on Learning Representations},
  pages={39745--39774},
  year={2024}
}

@INPROCEEDINGS{li2019deepgcns,
  author={Li, Guohao and Müller, Matthias and Thabet, Ali and Ghanem, Bernard},
  booktitle={IEEE/CVF International Conference on Computer Vision}, 
  title={{DeepGCNs}: Can {GCNs} Go As Deep As {CNNs}?}, 
  year={2019},
  volume={},
  number={},
  pages={9266-9275}}

@inproceedings{chen2020simple,
  title={Simple and Deep Graph Convolutional Networks},
  author={Chen, Ming and Wei, Zhewei and Huang, Zengfeng and Ding, Bolin and Li, Yaliang},
  booktitle={International Conference on Machine Learning},
  year={2020}
}

@article{poli2019graph,
  title={Graph neural ordinary differential equations},
  author={Poli, Michael and Massaroli, Stefano and Park, Junyoung and Yamashita, Atsushi and Asama, Hajime and Park, Jinkyoo},
  journal={arXiv preprint arXiv:1911.07532},
  year={2019}
}

@inproceedings{chamberlain2021grand,
  title={{GRAND}: Graph neural diffusion},
  author={Chamberlain, Ben and Rowbottom, James and Gorinova, Maria I and Bronstein, Michael and Webb, Stefan and Rossi, Emanuele},
  booktitle={International Conference on Machine Learning},
  pages={1407--1418},
  year={2021},
  organization={PMLR}
}

@inproceedings{platonov2023critical,
  title={A critical look at the evaluation of {GNNs} under heterophily: Are we really making progress?},
  author={Platonov, Oleg and Kuznedelev, Denis and Diskin, Michael and Babenko, Artem and Prokhorenkova, Liudmila},
  booktitle={International Conference on Learning Representations},
  year={2023}
}

@article{arroyo2025vanishing,
  title={On vanishing gradients, over-smoothing, and over-squashing in {GNNs}: Bridging recurrent and graph learning},
  author={Arroyo, {\'A}lvaro and Gravina, Alessio and Gutteridge, Benjamin and Barbero, Federico and Gallicchio, Claudio and Dong, Xiaowen and Bronstein, Michael and Vandergheynst, Pierre},
  journal={Advances in Neural Information Processing Systems},
  volume={38},
  pages={74356--74393},
  year={2026}
}

@article{levie2018cayleynets,
  title={CayleyNets: Graph convolutional neural networks with complex rational spectral filters},
  author={Levie, Ron and Monti, Federico and Bresson, Xavier and Bronstein, Michael M},
  journal={IEEE Transactions on Signal Processing},
  volume={67},
  number={1},
  pages={97--109},
  year={2018},
  publisher={IEEE}
}

@article{he2021bernnet,
  title={Bernnet: Learning arbitrary graph spectral filters via {Bernstein} approximation},
  author={He, Mingguo and Wei, Zhewei and Xu, Hongteng and others},
  journal={Advances in neural information processing systems},
  volume={34},
  pages={14239--14251},
  year={2021}
}

@article{dwivedi2022long,
  title={Long range graph benchmark},
  author={Dwivedi, Vijay Prakash and Ramp{\'a}{\v{s}}ek, Ladislav and Galkin, Michael and Parviz, Ali and Wolf, Guy and Luu, Anh Tuan and Beaini, Dominique},
  journal={Advances in Neural Information Processing Systems},
  volume={35},
  pages={22326--22340},
  year={2022}
}

@inproceedings{wang2022powerful,
  title={How powerful are spectral graph neural networks},
  author={Wang, Xiyuan and Zhang, Muhan},
  booktitle={International Conference on Machine Learning},
  pages={23341--23362},
  year={2022},
  organization={PMLR}
}

@article{barabasi1999emergence,
  title={Emergence of scaling in random networks},
  author={Barab{\'a}si, Albert-L{\'a}szl{\'o} and Albert, R{\'e}ka},
  journal={Science},
  volume={286},
  number={5439},
  pages={509--512},
  year={1999},
  publisher={American Association for the Advancement of Science}
}

@article{gu1996efficient,
author = {Gu, Ming and Eisenstat, Stanley C.},
title = {Efficient Algorithms for Computing a Strong Rank-Revealing {QR} Factorization},
journal = {SIAM Journal on Scientific Computing},
volume = {17},
number = {4},
pages = {848-869},
year = {1996},
}

@inproceedings{selvaraju2017grad,
  title={Grad-{CAM}: Visual explanations from deep networks via gradient-based localization},
  author={Selvaraju, Ramprasaath R and Cogswell, Michael and Das, Abhishek and Vedantam, Ramakrishna and Parikh, Devi and Batra, Dhruv},
  booktitle={Proceedings of the IEEE international conference on computer vision},
  pages={618--626},
  year={2017}
}

@article{jacques2015quantized,
  title={A quantized {Johnson--Lindenstrauss} lemma: The finding of {Buffon’s} needle},
  author={Jacques, Laurent},
  journal={IEEE Transactions on Information Theory},
  volume={61},
  number={9},
  pages={5012--5027},
  year={2015},
  publisher={IEEE}
}

@techreport{li1994secular,
    Author= {Li, Ren-Cang},
    Title= {Solving Secular Equations Stably and Efficiently},
    Year= {1994},
    Month= {Dec},
    Url= {http://www2.eecs.berkeley.edu/Pubs/TechRpts/1994/5882.html},
    Number= {UCB/CSD-94-851},
}

@book{pan2012structured,
  title={Structured matrices and polynomials: unified superfast algorithms},
  author={Pan, Victor Y},
  year={2012},
  publisher={Springer Science \& Business Media}
}

@article{levie2021transferability,
  title={Transferability of spectral graph convolutional neural networks},
  author={Levie, Ron and Huang, Wei and Bucci, Lorenzo and Bronstein, Michael and Kutyniok, Gitta},
  journal={Journal of Machine Learning Research},
  volume={22},
  number={272},
  pages={1--59},
  year={2021}
}

@article{fernandez2025int,
  title={INT-DTT+: Low-Complexity Data-Dependent Transforms for Video Coding},
  author={Fern{\'a}ndez-Mendui{\~n}a, Samuel and Pavez, Eduardo and Ortega, Antonio and Huang, Tsung-Wei and Canh, Thuong Nguyen and Su, Guan-Ming and Yin, Peng},
  journal={arXiv preprint arXiv:2511.17867},
  year={2025}
}

@book{jaja1992parallel,
  title={Parallel algorithms},
  author={J{\'a}J{\'a}, Joseph},
  year={1992}
}

@article{gasteiger2018predict,
  title={Predict then propagate: Graph neural networks meet personalized pagerank},
  author={Gasteiger, Johannes and Bojchevski, Aleksandar and G{\"u}nnemann, Stephan},
  journal={arXiv preprint arXiv:1810.05997},
  year={2018}
}

@inproceedings{giraldo2023trade,
  title={On the trade-off between over-smoothing and over-squashing in deep graph neural networks},
  author={Giraldo, Jhony H and Skianis, Konstantinos and Bouwmans, Thierry and Malliaros, Fragkiskos D},
  booktitle={Proceedings of ACM international conference on information and knowledge management},
  pages={566--576},
  year={2023}
}

@article{geisler2024spatio,
  title={Spatio-spectral graph neural networks},
  author={Geisler, Simon Markus and Kosmala, Arthur and Herbst, Daniel and G{\"u}nnemann, Stephan},
  journal={Advances in Neural Information Processing Systems},
  volume={37},
  pages={49022--49080},
  year={2024}
}

@article{stachenfeld2020graph,
  title={Graph networks with spectral message passing},
  author={Stachenfeld, Kimberly and Godwin, Jonathan and Battaglia, Peter},
  journal={arXiv preprint arXiv:2101.00079},
  year={2020}
}

@book{biyikougu2007laplacian,
  title={Laplacian eigenvectors of graphs: Perron-Frobenius and Faber-Krahn type theorems},
  author={Biyiko{\u{g}}u, T{\"u}rker and Leydold, Josef and Stadler, Peter F},
  publisher={Springer},
  year={2007}
}

@article{cuppen1980divide,
author = {Cuppen, J. J.},
title = {A divide and conquer method for the symmetric tridiagonal eigenproblem},
year = {1980},
issue_date = {June 1980},
publisher = {Springer-Verlag},
address = {Berlin, Heidelberg},
volume = {36},
number = {2},
journal = {Numer. Math.},
month = jun,
pages = {177–195},
numpages = {19}
}

@article{gu1995divide,
author = {Gu, Ming and Eisenstat, Stanley C.},
title = {A Divide-and-Conquer Algorithm for the Symmetric Tridiagonal Eigenproblem},
journal = {SIAM Journal on Matrix Analysis and Applications},
volume = {16},
number = {1},
pages = {172-191},
year = {1995}
}

@article{ou2022superdc,
author = {Ou, Xiaofeng and Xia, Jianlin},
title = {{SuperDC}: Superfast Divide-And-Conquer Eigenvalue Decomposition With Improved Stability for Rank-Structured Matrices},
journal = {SIAM Journal on Scientific Computing},
volume = {44},
number = {5},
pages = {A3041-A3066},
year = {2022}
}

@article{liang2025towards,
  title={Towards quantifying long-range interactions in graph machine learning: a large graph dataset and a measurement},
  author={Liang, Huidong and Borde, Haitz S{\'a}ez de Oc{\'a}riz and Sripathmanathan, Baskaran and Bronstein, Michael and Dong, Xiaowen},
  journal={arXiv preprint arXiv:2503.09008},
  year={2025}
}

@inproceedings{dwivedi2022graph,
  title={Graph Neural Networks with Learnable Structural and Positional Representations},
  author={Dwivedi, Vijay Prakash and Luu, Anh Tuan and Laurent, Thomas and Bengio, Yoshua and Bresson, Xavier},
  booktitle={International Conference on Learning Representations (ICLR)},
  year={2022}
}

@inproceedings{topping2022understanding,
  title={Understanding over-squashing and bottlenecks on graphs via curvature},
  author={Topping, Jake and Di Giovanni, Francesco and Chamberlain, Benjamin Paul and Dong, Xiaowen and Bronstein, Michael M},
  booktitle={International Conference on Learning Representations},
  year={2022}
}

\clearpage
\appendix
\crefalias{section}{appendix}
\crefalias{subsection}{appendix}
\thispagestyle{empty}

\onecolumn
\icmltitle{Supplementary Materials}
\section{Further definitions and notations}
\label{app:further}
This section formalizes conventions and assumptions that are used throughout the paper but may be implicit in the main exposition.

\subsection*{Secular equation}
Secular equations yield the eigenvalues of the system after a rank-one update, based on the eigenvalues of the matrix being updated and the update itself.
\begin{proposition}[Secular equation, \cite{golub1973some}]
    Let $\bt L = \bt U\diag{\pmb{\lambda}}\bt U^\top$ and $\tilde{\bt L} = \bt L + \rho \, \bt v\bt v^\top = \tilde{\bt U}\mathrm{diag}(\tilde{\pmb{\lambda}})\tilde{\bt U}^\top$. Then each $\tilde{\lambda}_j$ satisfies
    \begin{equation}
    \label{eq:secular}
    1 + \rho \, \sum_{i=1}^n \,  {z_i^2} \, / \, (\tilde{\lambda}_j - \lambda_i) = 0,
    \end{equation}
    where $\bt z = \bt U^\top \bt v$.
\end{proposition}

In practice, we use Li's secular equation solver \cite{li1994secular}. Eigenvalues before and after the update satisfy an interleaving property \cite{golub1973some}. 
\begin{proposition}[Eigenvalue interleaving \cite{bunch1978rank}]
\label{prop:interleaving}
Provided $\rho>0$, the eigenvalues after the update satisfy:
\begin{equation}
\label{eq:interleaving}
    \lambda_1 \leq \tilde{\lambda}_1 \leq \lambda_2 \leq \hdots \leq \lambda_n \leq \tilde{\lambda}_n. 
\end{equation}
\end{proposition}
This result guarantees that we can solve the secular equation efficiently \cite{gu1996efficient}, and can be used to control the complexity of the matrix vector product \cite{fernandez2025int} with Cauchy matrices.

\subsection*{Laplacian conventions}

We consider weighted undirected graphs with adjacency matrix $\mathbf{W}$ and self-loop matrix $\mathbf{V} \succeq 0$. The generalized graph Laplacian (GGL) is defined as \cite{biyikougu2007laplacian}:
\begin{equation}
\mathbf{L} \doteq \mathbf{D} - \mathbf{W} + \mathbf{V},
\end{equation}
where $\mathbf{D} = \mathrm{diag}(\mathbf{1}^\top \mathbf{W})$ is the degree matrix. Unless otherwise stated, all theoretical results are derived for the combinatorial Laplacian $\mathbf{L}$, but a similar derivation can be applied to the normalized Laplacian:
\begin{equation}
\mathbf{L}_{\mathrm{norm}} \doteq \mathbf{D}^{-1/2} \mathbf{L} \mathbf{D}^{-1/2},
\end{equation}
whose eigenvectors are orthonormal under the standard Euclidean inner product. The Cauchy factorization applies identically after normalization, up to a change of basis induced by $\mathbf{D}^{-1/2}$.

\subsection*{Eigenvalue ordering and degeneracies}

Let $\mathbf{L} = \mathbf{U} \mathrm{diag}(\boldsymbol{\lambda}) \mathbf{U}^\top$ denote an eigendecomposition of a symmetric Laplacian matrix. Eigenvalues are assumed to be sorted in non-decreasing order. In the presence of repeated eigenvalues, the corresponding eigenvectors are not unique; any orthonormal basis spanning the eigenspace is valid.

When rank-one updates leave certain eigenspaces invariant, we apply deflation (\cref{sup:deflation}) to isolate the affected subspace. This operation preserves orthogonality and does not alter the asymptotic complexity of the factorization.

\subsection*{Affected spectral indices}

Given a rank-one update $\tilde{\mathbf{L}} = \mathbf{L} + \rho \mathbf{v}\mathbf{v}^\top$ with
$\mathbf{L} = \mathbf{U} \mathrm{diag}(\boldsymbol{\lambda}) \mathbf{U}^\top$, define the projection vector
\begin{equation}
\mathbf{z} = \mathbf{U}^\top \mathbf{v}.
\end{equation}
When $\mathbf{L}$ has distinct eigenvalues, we define the set of affected spectral indices as
\begin{equation}
\mathcal{S} \doteq \{ i \mid z_i \neq 0 \}.
\end{equation}

When $\mathbf{L}$ has repeated eigenvalues, the associated eigenspaces are not uniquely defined. In this case, we exploit the rotational freedom within each eigenspace and choose an orthonormal eigenbasis such that, for each eigenspace, at most one basis vector has a nonzero projection onto $\mathbf{v}$ \cite{bunch1978rank}. This choice is always possible and leaves the eigendecomposition of $\mathbf{L}$ unchanged. The definition of $\mathcal{S}$ then applies unchanged, with $\mathbf{U}$ denoting the chosen eigenbasis.

Only eigenpairs indexed by $\mathcal{S}$ are modified by the rank-one update; eigenvectors corresponding to indices outside $\mathcal{S}$ remain invariant. For bridge edges connecting two subgraphs, $\mathcal{S}$ is restricted to spectral components associated with those subgraphs, yielding localized Cauchy factors.

\subsection*{Rank-one versus rank-$k$ updates}

Each individual edge insertion corresponds to a rank-one update of the Laplacian. Cuts consisting of $k$ bridge edges therefore induce a sequence of $k$ rank-one updates, each associated with its own Cauchy factor. The resulting transformation is given by the product of these factors, yielding a rank-$k$ modification of the spectrum.

All theoretical results are stated for rank-one updates and extend directly to rank-$k$ cuts by composition and induction.

\subsection*{Spectral locality}

We say that a linear operator $\mathbf{A} \in \mathbb{R}^{n \times n}$ is \emph{$s$-local in the spectral domain} if it acts non-trivially on at most $s$ spectral coordinates, i.e., if
\begin{equation}
\mathbf{A} = \mathbf{P}^\top
\begin{bmatrix}
\mathbf{I}_{n-s} & \mathbf{0} \\
\mathbf{0} & \mathbf{B}
\end{bmatrix}
\mathbf{P},
\end{equation}
for some permutation matrix $\mathbf{P}$ and $\mathbf{B} \in \mathbb{R}^{s \times s}$.

Cauchy factors induced by bridge edges are $O(|V_i| + |V_j|)$-local, where $V_i$ and $V_j$ are the vertex sets of the connected subgraphs. In contrast, the full GFT basis is $n$-local. This distinction underpins the locality–globality trade-off exploited by L2G-Net.

\subsection*{Non-uniqueness of hierarchical decompositions}

The hierarchical graph family (HGF) decomposition of a graph is generally not unique. Different choices of partitions and merge orders yield different factorizations with identical spectral semantics. Our objective is not to recover a unique hierarchy, but to identify decompositions that minimize computational cost while preserving spectral equivalence.

\subsection*{Graph sparsification}
Graph sparsification seeks to approximate a graph $\mathcal{G}$ with a sparser graph $\tilde{\mathcal{G}}$ whose Laplacian remains spectrally close to that of $\mathcal{G}$. Let $\mathbf{L}$ and $\tilde{\mathbf{L}}$ denote the corresponding graph Laplacians. A $(1\pm\varepsilon)$ spectral sparsifier satisfies
\[
(1-\varepsilon)\,\mathbf{x}^\top \mathbf{L}\mathbf{x}
\;\le\;
\mathbf{x}^\top \tilde{\mathbf{L}}\mathbf{x}
\;\le\;
(1+\varepsilon)\,\mathbf{x}^\top \mathbf{L}\mathbf{x},
\qquad \forall\,\mathbf{x}\in\mathbb{R}^n.
\]

such a sparsifier can be constructed by sampling edges proportionally to their effective resistances, yielding $\tilde{\mathcal{G}}$ with $O(n\log n/\varepsilon^2)$ edges and preserving the spectrum of $\mathbf{L}$ up to relative error \cite{spielman2008graph}. In this work, we use spectral sparsification in a localized manner, applying these guarantees to inter-subgraph cuts. This reduces the number of connecting edges while preserving the quadratic form of the global Laplacian.

\section{Related work}
\label{app:related_work}
\textbf{Graph rewiring and transformers.}
Several approaches address long-range dependencies by modifying the graph structure or adding attention-based mechanisms. Graph rewiring methods aim to reduce the effective graph diameter \cite{barbero2023locality, gasteiger2018predict, giraldo2023trade}, while graph transformers model global interactions through attention mechanisms \cite{deng2024polynormer}. However, these approaches often add computational complexity by creating denser graph shift operators or weakening the graph's topological inductive bias \cite{hariri2025return, giraldo2023trade} (cf.~\cref{sec:gradcam}). In contrast, our Cauchy factorization is exact; potential  approximations (spectral sparsification) simplify the computational complexity while providing guarantees on the spectrum of the resulting Laplacian \cite{spielman2008graph}

\textbf{Positional encodings (PEs).}
Using Laplacian eigenvectors as node features has been proposed to inject global structural information into MPNNs \cite{dwivedi2020generalization, huang2023stability}. PEs enable global interactions, but they do so by allowing information in nodes that are far away (in the graph) to be combined. In contrast, L2G-Net enables long-range operations by composing local spectral operators across the hierarchy (subgraph-to-subgraph). These mechanisms are complementary: our work shows that combining L2G-Net with PEs further improves performance. PEs are typically identified with spectral methods, while we use spectral methods to identify subgraphs and create a hierarchical, graph-adaptive, task-independent architecture.

\textbf{Deep GNNs and differential equations.} Deep architectures \cite{li2019deepgcns} or Neural ODEs \cite{poli2019graph, chamberlain2021grand} can also help expand the receptive field via techniques like residual connections \cite{chen2020simple}. Nonetheless, these methods still rely on local message passing to propagate information; so the receptive field grows at most linearly with depth, often requiring hundreds of layers to capture long-range interactions. In contrast, our hierarchical spectral framework decouples the receptive field from depth. This allows the architecture to focus on learning complex feature transformations without the signal degradation or over-smoothing issues \cite{rusch2023survey}  typical of vertex-domain propagation.

\chan{
\textbf{Hybrid Spatial–Spectral GNNs.} \cite{stachenfeld2020graph} compute a truncated GFT and learn a MPNN in a fully connected “spectral graph” whose vertices correspond to the GFT coefficients. This approach relies on partial eigendecomposition and operates on a fixed set of low-frequency components. Similarly, \cite{geisler2024spatio} interleave vertex-domain polynomial message passing with truncated spectral filters, combining spatial operators with projections onto a limited set of eigenvectors. In both cases, spectral processing is confined to a subset of the spectrum. In contrast, L2G-Net neither truncates the spectrum nor approximates the graph Fourier transform. We show that the full GFT admits an exact hierarchical factorization via Cauchy matrices associated with interface edges. This enables global spectral processing without computing a dense eigendecomposition. Locality in L2G-Net is induced by graph partitioning and merging, rather than through polynomial filtering in the vertex domain.
}

\chan{
\textbf{Divide-and-conquer eigensolvers}. The structure of eigenvector matrices arising from rank-one updates of symmetric eigenproblems was first analyzed by \cite{bunch1978rank}. Divide and conquer eigensolvers were first proposed in \cite{cuppen1980divide, gu1995divide}. This line of work has since been extended beyond tridiagonals to dense matrices admitting hierarchically semiseparable (HSS) representations \cite{ou2022superdc}, where the input matrix must first be compressed into HSS form by finding low-rank off-diagonal blocks, which are often approximated to a user-specified tolerance. Our setting differs in both its input and its output. First, the input is a sparse symmetric matrix specified by a graph, so hierarchical structure is given by the graph itself, making the resulting factorization exact rather than approximate. As a result, our complexity is stated in terms of a combinatorial property of the graph (its cut size) rather than a numerical off-diagonal rank. Second, the output of our method is not the eigendecomposition; instead, we obtain a Cauchy factorization of the eigenbasis, which is a structural statement about the GFT basis. In this factorization, each Cauchy factor corresponds to a specific bridge edge between subgraphs, decomposing the basis as a product of mixing operators between identifiable graph regions. This is what enables L2G-Net's local to global architecture, where learnable filters slot in between Cauchy factors corresponding to specific cuts.
}

\textbf{Fast algorithms for GFTs.} Early works proposed approximating the Laplacian \cite{magoarou2018approx} or its eigenspace \cite{frerix2019approximating} using sequences of sparse or structured operators, such as Givens rotations or butterfly-like factorizations, aiming to reduce the complexity relative to dense eigendecomposition. However, spectral-mismatch errors with these methods are difficult to control \emph{a priori}. 
While \cite{fernandez-menduina2025fast} established the link between rank-one updates and spectral graph theory, the resulting factorization of \eqref{eq:icassp} was valid only under some restrictions on the eigenvalues of $\bt L$ and on the updates. 
Moreover, the work  \cite{fernandez-menduina2025fast} focused on incremental rank-one updates of path graphs and did not consider the problem of decomposing an arbitrary graph to facilitate the factorization of its GFT.

\section{Experimental details}
\label{app:arch_details}
The hyperparameters reported in Table~\ref{tab:hyperparams} specify the architectural, spectral, and optimization settings used for L2G-Net across datasets.

The hidden dimension $d$ denotes the dimensionality of node embeddings throughout the network. \emph{Layers} corresponds to the number of stacked blocks, each consisting of a spectral filtering stage followed by a shared feed-forward module with residual connections. 

The parameter $K$ denotes the number of spectral filters per filter bank, which determines the resolution at which the graph spectrum is sampled. \chan{When dealing with inductive tasks, since the orientation of the graphs is arbitrary, the banks for the same level of hierarchy are shared across all subgraphs. For all datasets, filters are expressed using B-spline bases defined over the normalized Laplacian eigenvalues, which encourage smooth spectral responses.}

\emph{Dropout} denotes the feature dropout probability applied to node representations during training. \emph{LR} is the base learning rate used by the AdamW optimizer. \emph{Steps} denotes the total number of optimization steps performed during training.

Finally, \emph{Sparsif.} refers to the cut sparsification ratio used during graph partitioning. In particular, in practice, we can interpret the parameter as using $\max(1, \text{edges after}/\text{edges before} \times 100)$ edges. This parameter controls the fraction of cut edges retained via effective resistance sampling, trading off computational cost and approximation accuracy in the spectral decomposition. A higher value retains more edges across the partition boundary, while smaller values lead to more aggressive sparsification.

Cross-partition edges are sparsified using resistance-based sampling. The partition is obtained via a Fiedler vector computed with LOBPCG ($k{=}2$, $40$ iterations), using a quantile sweep over $[0.45,0.55]$. Effective resistances are approximated with a JL \cite{jacques2015quantized} with dimension $k=\lceil 24\log |V|/\varepsilon^2\rceil$ using $\varepsilon=0.5$, with a minimum of $k=20$. Crossing edges are sampled with replacement according to probabilities proportional to $w_e \widetilde{R}_e$, where $\widetilde{R}_e$ is the approximate resistance. A fraction $\rho$ of crossing edges is retained, depending on the dataset, and sampled edges are reweighted by $(q p_e)^{-1}$. For details regarding the performance of the sparsifier, we refer the reader to \cite{spielman2008graph}.

\begin{table}[t]
\centering
\caption{Dataset-specific hyperparameter settings for L2G-Net.}
\label{tab:hyperparams}
\small
\setlength{\tabcolsep}{6pt}
\renewcommand{\arraystretch}{1.1}
\begin{tabular}{lcccccccc}
\toprule
\textbf{Dataset}
& $d$
& Layers
& Levels 
& $K$
& Dropout
& LR
& Steps
& Sparsif. \\
\midrule
\textit{Roman-empire}
& 64
& 5
& 1
& 8
& 0.25
& $2\times10^{-4}$
& 1000
& 0.005 \\

\textit{Amazon-ratings}
& 256
& 4
& 1
& 6
& 0.30
& $4\times10^{-3}$
& 1500
& 0.005 \\

\textit{Minesweeper}
& 32
& 14
& 1
& 6
& 0.225
& $2\times10^{-3}$
& 1000
& 0.005 \\

\textit{Tolokers}
& 64
& 3
& 1
& 4
& 0.20
& $8\times10^{-4}$
& 2000
& 0.01 \\

\midrule
\textit{peptides-func} & 256 & 4 & 1 & 12 & 0.35 & $1\times 10^{-3}$ & 500 & 1\\
\textit{peptides-struct} & 256 & 4 & 1 & 12 & 0.3 & $1\times 10^{-3}$ & 300 & 1\\
\midrule
\textit{Paris} & 64 & 16 & 4 & 12 & 0.3 & $5\times 10^{-3}$ & 5000 & 0.150\\
\textit{Shanghai} & 64 & 16 & 4 & 12 & 0.35 &  $5\times 10^{-3}$ & 5000 & 0.130 \\
\textit{Los Angeles} & 64 & 16 & 5 & 12 & 0.25 & $2.5\times 10^{-3}$ & 5000 & 0.100 \\
\textit{London} & 64 & 16 & 6 & 12 & 0.4 & $2\times 10^{-3}$ & 5000 & 0.125 \\
\bottomrule
\end{tabular}
\end{table}

\subsection{Base layer formulation}
Following \cite{hariri2025return}, we adopt an ODE-inspired formulation. Let $\bt X(t)$ denote the node features at continuous time $t$. We let $\bt U(\bt \Phi)$ be the local to global forward transform with local filtering in our method. Then,
\begin{equation}
    \mathrm{d}\bt X(t) = \bt U g_\theta(\pmb \lambda) \bt U(\bt \Phi)^\top \bt X(t)\bt W\, \mathrm{d} t.
\end{equation}
Applying an explicit Euler discretization with step size $\epsilon > 0$ yields the update for layer $m = 0, \dots, M-1$:
\begin{equation}
    \bt X_{m+1} = \bt X_{m} + \epsilon\, \bt U g_\theta(\pmb \lambda) \bt U(\bt \Phi)^\top \bt X_{m}\bt W.
\end{equation}
Depth allows for compositionality (learning complex non-linearities) rather than to expand the receptive field, which is already global due to the spectral formulation. Unlike \cite{hariri2025return}, the complexity of the spectral filter does not directly control the norm of the layer Jacobian in our model. To reduce the parameter count, we share $\bt W$ and the spectral filter parameters $\theta, \bt \Phi$ across layers. Furthermore, we also share the spectral filter across channels.
 
\subsection{Real-world graph properties}
We provide some information about the properties of the graphs we consider in \cref{tab:platonov_stats}. We also include the cut size after sparsification during our experiments.

\begin{table}[t]
\centering
\caption{Statistics of real-world heterophilous graph datasets from Platonov \textit{et al.}~\cite{platonov2023critical}.}
\label{tab:platonov_stats}
\begin{tabular}{lrrrrrr}
\toprule
Dataset & Nodes & Edges & Avg. Degree & Features & Classes & $k$ after sparse \\
\midrule
\emph{Roman-empire}     & 22{,}662 & 32{,}927  & 2.91  & 300 & 18 & 1\\
\emph{Amazon-ratings}   & 24{,}492 & 93{,}050  & 7.60  & 300 & 5  & 1\\
\emph{Minesweeper}      & 10{,}000 & 39{,}402  & 7.88  & 7   & 2  & 1\\
\emph{Tolokers}         & 11{,}758 & 519{,}000 & 88.28 & 10  & 2  & 25\\
\bottomrule
\end{tabular}
\end{table}

\section{Deflation}
\label{sup:deflation}
The derivation in \cite{fernandez-menduina2025fast}, which is generalized in our work, assumes a worst-case scenario\footnote{Worst case in terms of complexity: applying deflation simplifies the factorization.} where 1) the GGL before the rank-one update has no repeated eigenvalues and 2) the rank-one update used to add an edge is not orthogonal to any element in the original bases of the eigenspace. 
In these cases, we first deflate the matrix, reducing the computational complexity of the factorization. Following \cite{bunch1978rank},  consider the rank-one update:
\begin{equation}
\newL = \baseL + \rho \, \bt v \bt v^\top = \newU\diag\bfnew\newU^\top
\end{equation}
where $\baseL = \bt U\diag\bfbase \bt U^\top$ is the original eigendecomposition. Let $\bt z = \bt U^\top \bt v$ be the projection of the perturbation vector onto the eigenspace of $\baseL$. Then, we can write: $\newL = \bt U (\diag{\bfbase} + \rho \, \bt z \bt z^\top) \bt U^\top$. The idea of deflation is that, in some scenarios, we can find a set of eigenvectors of $\bt L$ (potentially different than $\bt U$) such that some of its elements are orthogonal to the original rank-one update $\bt v$, so that $\bt z$ has some zero entries. Since $\bt U$ in our case is given by our algorithm, we differentiate two cases: direct deflation (case 1) and deflation after applying a Householder reflector (case 2).

\paragraph{Case 1: Zero components} ($z_i = 0$, for $i = 1, \hdots, \ell$). The perturbation is then orthogonal to the basis of the original eigenspace. We obtain $\tilde{\lambda_i} = \lambda_i$ and $\tilde{\bt u}_i = \bt u_i$:
\begin{equation}
\newL \bt u_i = \bt U(\diag{\bfbase} + \rho \, \bt z\bt z^\top)\bt e_i = \bt U(\diag{\bfbase} \bt e_i) = \lambda_i \bt u_i,
\end{equation}
where $\bt e_i$ is the $i$th element of the canonical basis.

\textbf{Deflation:} Remove column and row $i$ from $(\diag{\bfbase} + \rho \, \bt z \bt z^\top)$. Construct the Cauchy matrix based on the updated secular equation. During multiplication, let $\bt u_i^\top\bt x$ be the $i$th coefficient, i.e., let it pass through unmodified.

\textbf{Effect on complexity:} Reduces the problem size from $n$ to $n-\ell$, where $\ell$ is the number of zero components.

\paragraph{Case 2: Repeated eigenvalues.} Consider a repeated eigenvalue $\lambda_k$ with multiplicity $m$, and let $\mathcal{S}_k = \{i : \lambda_i = \lambda_k\}$ denote the index set of all eigenvectors corresponding to $\lambda_k$. The basis for this eigenspace is not unique; any orthogonal transformation of $\{ \bt u_i \}_{i \in \mathcal{S}_k}$ yields another valid basis. We can exploit this freedom to choose a new basis that zeroes out components of $\bt z$. Let $\bt z_k$ be the $m$-dimensional subvector of $\bt z$ corresponding to the indices in $\mathcal{S}_k$. We can find an $m \times m$ orthogonal matrix $\bt Q$ (e.g., a Householder reflector) such that:
\begin{equation}
\bt Q^\top \bt z_k = \|\bt z_k\|_2 \bt e_1.
\end{equation}
Applying this rotation to the basis of the eigenspace results in a new set of eigenvectors where the transformed perturbation vector, $\hat{\bt z}$, has $m-1$ zero components in the subspace. Since the eigenvalues are identical within this subspace ($\lambda_i = \lambda_k \ \forall i \in \mathcal{S}_k$), the diagonal matrix $\diag\bfbase$ is invariant under this rotation. This effectively reduces the problem to Case 1 for $m-1$ eigenpairs, as they become orthogonal to the perturbation.

\textbf{Deflation:} Construct an orthogonal matrix $\bt Q$ that transforms the subvector $\bt z_k$ as shown above. Let $\hat{\bt U} = \bt U_{\mathcal{S}_k} \bt Q$. This rotation leaves $m-1$ eigenpairs, $(\lambda_k, \hat{\bt u}_j)$, unaffected by the rank-one update. We can then remove the $m-1$ rows and columns of $(\diag{\bfbase} + \rho \, \hat{\bt z} \hat{\bt z}^\top)$ corresponding to the positions where $\hat{\bt z}$ is zero and solve the smaller problem.

\textbf{Effect on complexity:} Reduces the problem size from $n$ to $n-(m-1)$.

\chan{Near-degenerate eigenvalues are common in practice: for instance, on \textit{Roman Empire}, about 4.5\% of eigengaps fall below our threshold of $10^{-12}$. Deflation handles these cases without degrading factorization accuracy or downstream performance. We show the eigengaps (the difference between consecutive eigenvalues) for the four datasets in \cite{platonov2023critical} in \cref{fig:eigengaps}.}

\begin{figure}
    \centering
    \includegraphics[width=0.49\linewidth]{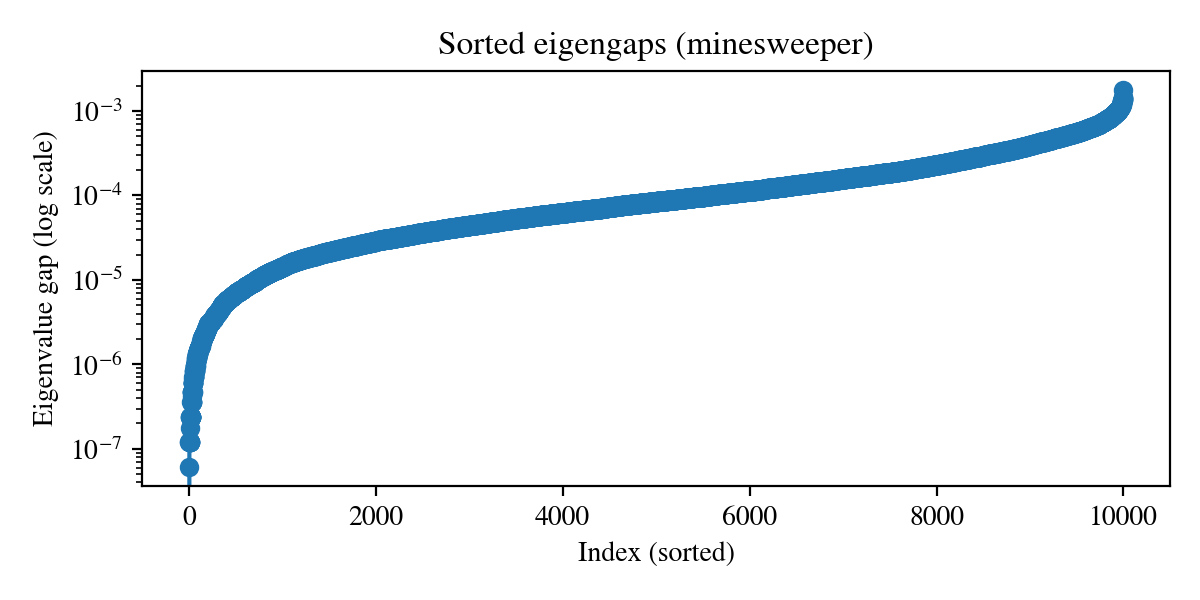}
\includegraphics[width=0.49\linewidth]{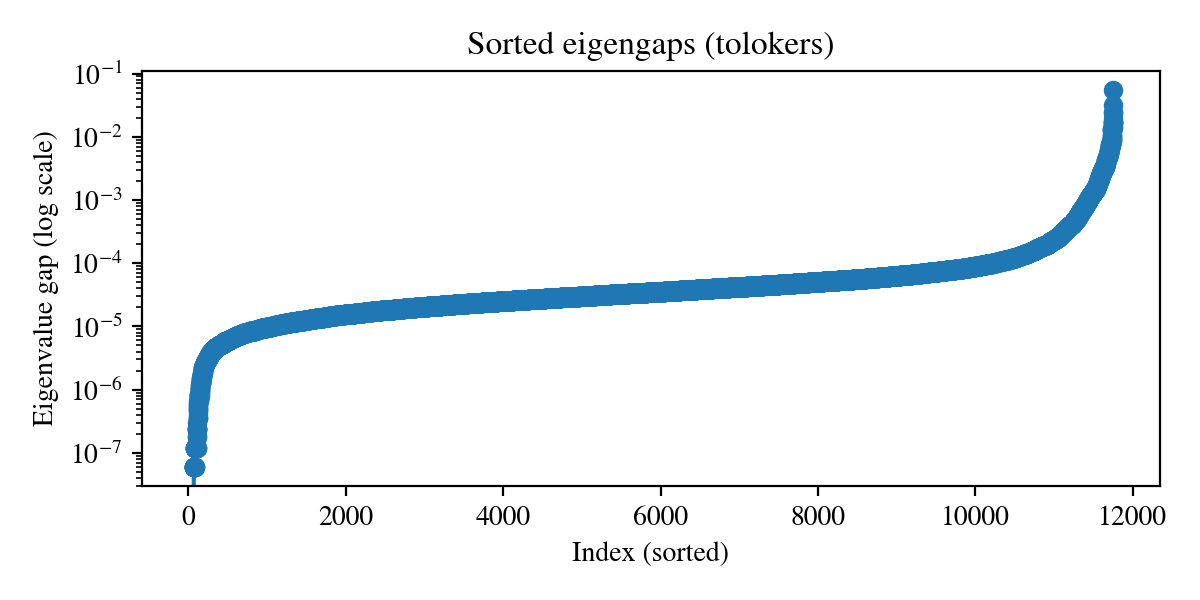}

\includegraphics[width=0.49\linewidth]{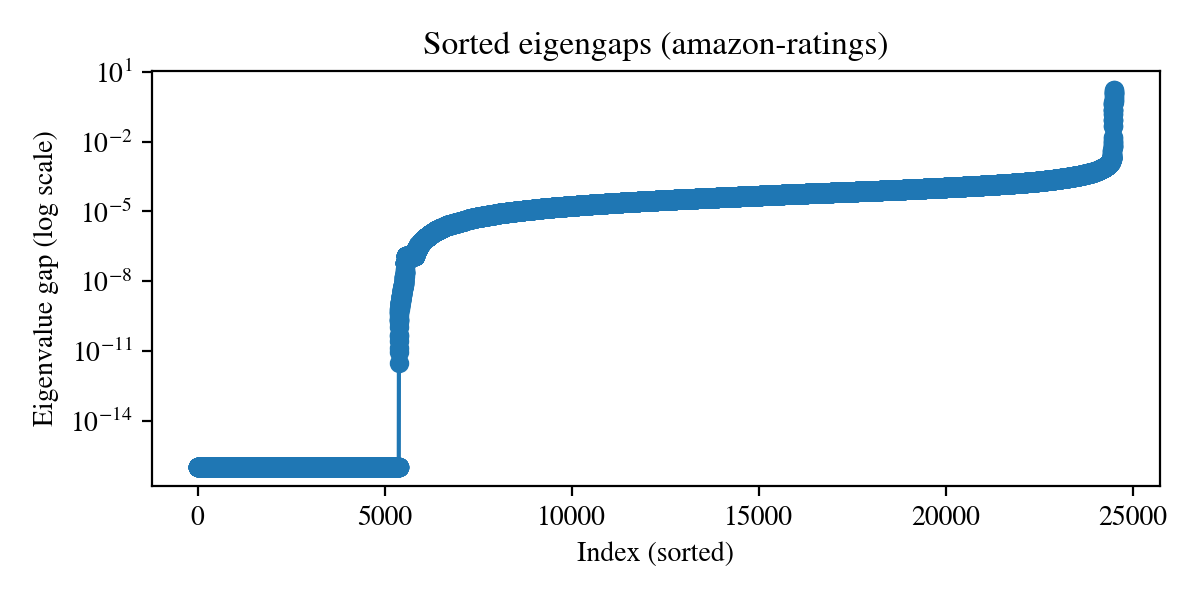}
\includegraphics[width=0.49\linewidth]{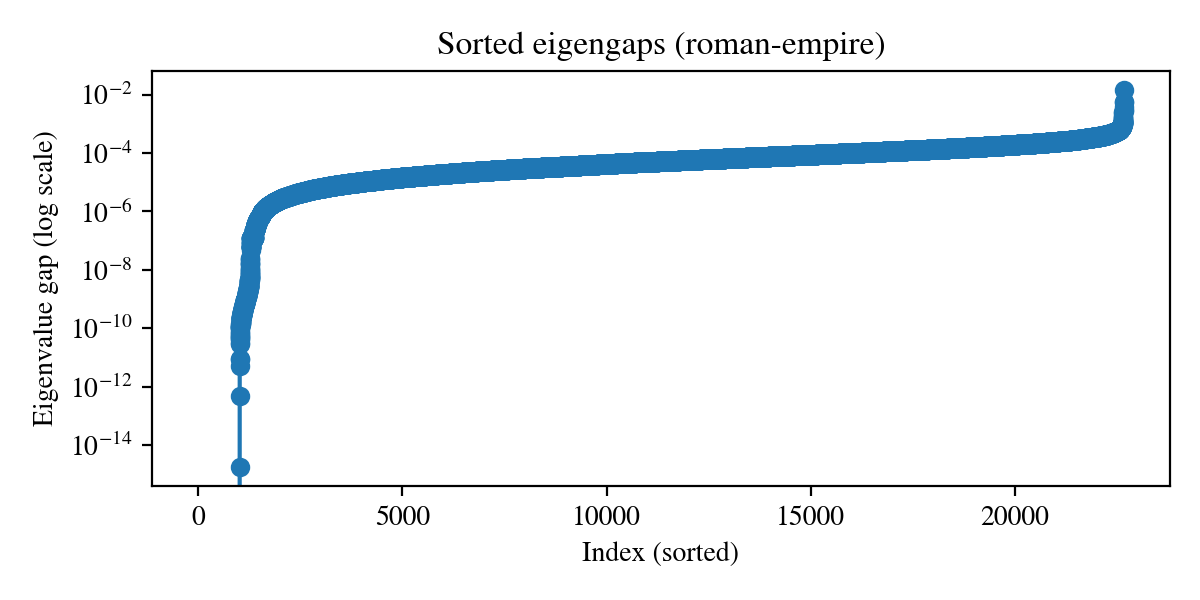}  
    \caption{\chan{Sorted eigengaps of the normalized Laplacian on the heterophilous benchmarks of \cite{platonov2023critical}. \emph{Amazon-ratings} exhibits a large plateau of numerically zero gaps (approximately 20\% of eigenvalues), and \emph{Roman-empire} shows a smaller cluster of gaps below $10^{-12}$; \emph{Minesweeper} and \emph{Tolokers} are well-separated throughout. Deflation handles near-degenerate eigenvalues in all cases without affecting downstream accuracy.}}
    \label{fig:eigengaps}
\end{figure}

\section{Proof of \cref{thm:cauchy_factorization}}
\label{app:proof_master}
We show that the Laplacian eigenvectors of any graph in $\mathcal{F}(L, \lbrace \mathcal{G}_i\rbrace_{i = 1}^m, k)$ can be written as a chain product of CF. We denote by $\bt x_i$ the subset of $\bt x$ corresponding to the subset of nodes of $\mathcal{G}$ represented by $\mathcal{G}_i$.

\subsection{CF decomposition}
First, we establish the following property of any GGL, which states that the Laplacian decomposes into rank-one updates, one for each edge  contribution.
\begin{proposition}[\cite{batson2014twice}]
\label{prop:rank_one_laplacians_sup}
    Let $\bt e_j$ be the $j$th canonical vector, for $j = 1, \hdots, n$. Then, the GGL $\bt L$ of an undirected graph can be written as
    \begin{equation}
        \bt L = \sum_{{(i, j)} \in  \mathcal{E}, \, i\not=j}\, w_{ij} (\bt e_i - \bt e_j)(\bt e_i - \bt e_j)^\top
        + \sum_{(i, i)\in\mathcal{E}}\,  w_{ii} \bt e_i\bt e_i^\top,
    \end{equation}
    where $w_{ij}$ denotes the edge weights, for $i, j = 1, \hdots, \vert \mathcal{V} \vert$. 
    \end{proposition}

We can state now the following result, which generalizes \cite{fernandez-menduina2025fast} to arbitrary symmetric matrices.
\begin{theorem}[Progressive decomposition identity]
\label{thm:main_eq}
    The eigenvector basis $\tilde{\bt U}^\top$ of $\tilde{\bt L} = \tilde{\bt U}\diag{\tilde{\pmb{\lambda}}}\tilde{\bt U}^\top =  \bt L + \alpha \bt v\bt v^\top$, with $\bt L = \bt U\diag{\pmb{\lambda}}\bt U^\top$, decomposes as
    \begin{equation}
        \label{eq:fwd_map_sup}
        \tilde{\bt U}^\top = - \Dml \bt U^\top.
    \end{equation}    
\end{theorem}
\begin{proof}
When $\mathcal{S}$ is empty, the result boils down to the progressive factorization in \cite{fernandez-menduina2025fast}. When $\mathcal{S}$ is not empty, we apply deflation first, and construct $\bt C(\tilde{\pmb{\lambda}_{\mathcal{S}}}, \pmb{\lambda}_{\mathcal{S}})$ from the deflated matrix.
\end{proof}
When $\newL$ is the Laplacian of a graph, this result offers a \emph{progressive factorization} of the Laplacian eigenvectors:  we can decompose the new basis as the eigenvectors of the unperturbed Laplacian multiplied by a CF. 

Then, we can design a decomposition that starts with a fully disconnected graph and progressively adds edges. Let $\bfbase_0, \hdots, \bfbase_{\vert \mathcal{E} \vert}$ be the sequence of eigenvalues of the GGL of the graph obtained by performing each step of this process. Then, we can state the following result. 
\begin{lemma}[CF decomposition]
The Laplacian eigenvectors of the GGL $\bt L = \bt U \diag\bfbase\bt U^\top$ of any graph $\mathcal{G}$ can be written as the product of $\vert \mathcal{E} \vert$ CF:
\begin{equation}
\bt U^\top = \bt D(\bfbase_{\vert \mathcal{E} \vert}, \bfbase_{\vert \mathcal{E} \vert-1})\hdots \bt D(\bfbase_{2}, \bfbase_{1})\bt D(\bfbase_{1}, \bfbase_{0}).
\end{equation}
\end{lemma}
\begin{proof}
The result follows from applying \cref{thm:main_eq} accounting for every possible rank-one update in \cref{prop:rank_one_laplacians}. 
\end{proof}
Now, we can use the previous result to state a decomposition for graphs in $\mathcal{F}(L, \lbrace \mathcal{G}_i\rbrace_{i = 1}^m, k)$. We define the stack of matrices of base transforms as:
\begin{equation}
    \bt U_b^\top \doteq 
    \begin{bmatrix}
    \bt U_1^\top & \bt 0 &  \hdots & \bt 0 \\
    \bt 0 & \bt U_2^\top & \hdots & \bt 0\\
    \vdots & \ddots & \ddots & \vdots \\ 
    \bt 0 & \bt 0 &\hdots & \bt U_m^\top
    \end{bmatrix}.
\end{equation}
Let $\bfbase$ be the eigenvalues of the GGL for the graph with the isolated base components. We then denote by $\bfnew_{i}$, for $i = 1, \hdots, k\, (2^L - 1)$, the eigenvalues resulting from each rank-one update.
\begin{lemma}[HGF decomposition]
\label{lemma:hgf_decomposition}
    Let $\tilde{\bt U}$ be Laplacian eigenvectors of any graph $\mathcal{G} \in \mathcal{F}(L, \lbrace \mathcal{G}_i\rbrace_{i = 1}^m, k)$. Then, 
    \begin{equation}
        \tilde{\bt U}^\top =  \bt D(\bfnew_{k(2^L - 1)}, \bfnew_{k(2^L - 1)-1}) \hdots \bt D(\bfnew_1, \bfbase)\bt U_b^\top.
    \end{equation}
\end{lemma}
\begin{proof}
The result follows by induction from \cref{prop:rank_one_laplacians} and \cref{thm:main_eq}.
\end{proof}
Although the decomposition consists of $k(2^L - 1)$ products, in a HGF each CF modifies only a reduced subset of the input vector, which can be exploited to systematically reduce complexity. The next result shows that CF are localized only to the subgraphs they modify.
\begin{proposition}[Edge connection]
\label{prop:cf_to_ocm}
    Consider a graph $\mathcal{G}$ with GGL $\bt L$ formed by a set of $m$ disconnected subgraphs $\mathcal{G}_i$, for $i = 1, \hdots, m$. Let $\tilde{\bt L} = \bt L + \rho \, \bt v\bt v^\top$ and $n = \vert \mathcal{V}\vert$. Then,
    \begin{equation}
        \bt D(\bfnew, \bfbase) = 
           \bt P_{\mathcal{S}}^\top  \begin{bmatrix}
                \bt I_{n-s\times n-s} & \bt 0 \\
                \bt 0 & -\bt C(\bfnew_{\mathcal{S}}, \bfbase_{\mathcal{S}})
            \end{bmatrix} \bt P_{\mathcal{S}}\in \mathbb{R}^{n\times n}.
    \end{equation}
    with 
    \begin{itemize}
        \item $s = \vert \mathcal{V}_i\vert + \vert\mathcal{V}_j\vert$ when the rank-one update connects $\mathcal{G}_i$ to $\mathcal{G}_j$, with $i \not = j$,
        \item $s = \vert\mathcal{V}_i\vert$ when the rank-one update affects only the subgraph $\mathcal{G}_i$.
    \end{itemize}
\end{proposition}
\begin{proof}
From the progressive property \cref{thm:main_eq} and the definition of CF, we have to show that the transform version of the rank-one update $\bt U_b^\top \bt v$ has at least $n-s$ zeros. Since
\begin{equation}
    \bt z = \bt U_b^\top\bt v =     \begin{bmatrix}
    \bt U_1^\top & \bt 0 &  \hdots & \bt 0 \\
    \bt 0 & \bt U_2^\top & \hdots & \bt 0\\
    \vdots & \ddots & \ddots & \vdots \\ 
    \bt 0 & \bt 0 &\hdots & \bt U_m^\top
    \end{bmatrix} \bt v =  \begin{bmatrix}
    \bt U_1^\top \bt v_1 \\
    \bt U_2^\top \bt v_2 \\
    \bt U_m^\top \bt v_m
    \end{bmatrix},
\end{equation}
where $\bt v_i$ denotes the subsection of the vector $\bt v$ that corresponds to the $i$th subgraph. Now, for the first case, $\bt v$ only has two non-zero components, which will be in $\bt v_i$ and $\bt v_j$. Therefore, $\bt U_k^\top \bt v_k$ is zero for all $k \not = i, j$. This set comprises $n - s$ components, with $s = \vert \mathcal{V}_i\vert + \vert \mathcal{V}_j\vert$. Therefore, the first result follows.

For the second case, both non-zero components will lie in $\bt v_i$. Therefore, $\bt U_k^\top \bt v_k$ is zero for all $k \not = i$. This set comprises $n - s$ components, with $s = \vert \mathcal{V}_i\vert$. Therefore, the result follows.
\end{proof}
The previous result states that when we join two distinct subgraphs $\mathcal{G}_i$ and $\mathcal{G}_j$ via a rank-one update, the corresponding CF is an identity matrix except for a block of size $s = \vert \mathcal{V}_i\vert + \vert\mathcal{V}_j\vert$. This localizes the update to only the dimensions corresponding to the vertices of those two subgraphs. In the next section, we provide a method to implement these structured matrix products efficiently.

\section{Constructing the factorization}
\label{app:cauchy_eigen}
We detail our algorithm for eigendecomposition in \cref{alg:cauchy_full}. 
We analyze its complexity as follows:

\begin{itemize}

    \item \textbf{Apply transform.} 
    Since we add $k$ edges between pairs of subgraphs, the maximum number of 
    edges touching any given subgraph is $kL$.    Therefore, for a matrix $\mathbf{M}\in\mathbb{R}^{s\times s}$, the complexity 
    of each call is 
    \[
        O(g(s)\, k\, L),
    \]
    where $g(s)$ is the complexity of the matrix-vector product with the input matrix. This bound is a worst-case bound over the full hierarchy; tighter level-wise bounds are used in the merging analysis
    \item \textbf{Initialization.}
    Creating and updating each element of $\mathcal{Z}$ is linear in the size 
    of the smallest graphs, $n/2^L$. Since we have at most 
    $
        k\sum_{\ell=1}^L 2^{L-\ell}
        = k(2^L - 1)
    $
    elements in $\mathcal{Z}$, the cost is upper bounded by $O(nk)$.

    The eigendecomposition of the leaf graphs has cost
    \[
        O\!\left( \sum_{i = 1}^{m} \, f_i(n)\right).
    \]

    Updating all elements of $\mathcal{Z}$ across all levels requires
    \[
        k(2^L - 1)  O\!\left(n/2^L\right) = O(kn).
    \]
    While we are multiplying by dense matrices, multiplying by each element of the canonical basis amounts to selecting the corresponding column of the dense matrix. Since we have to scale by the weight, the complexity becomes $2 \, n/2^L$, which yields the result.
    
    \item \textbf{Hierarchical merging.}
    Solving the secular equation at level $\ell$ costs 
    $
        O((n/2^{L-\ell})^2)
    $. Since at that level, at most $k\,2^{L-\ell}$ edges participate, this operation has complexity
    $$ 
    \sum_{\ell=1}^L 
        k\,2^{L-\ell}\, O((n/2^{L-\ell})^2) = O(kn^2).
    $$

    We must also apply the corresponding Cauchy updates. The complexity of an update at level $\ell$ is
    $
        O(g(n/2^{L-\ell})\, k(L-\ell))
    $ 
    per edge, since each edge update requires a multiplication by all the other incoming updates that touch the corresponding subgraph.  
    Therefore the total merging complexity is
    \[
    \sum_{\ell=1}^L 
        k\,2^{L-\ell}\, O\!\left(k(L-\ell)g(n/2^{L-\ell})\, \right)
      = 
        O\!\left( k^2
            \sum_{\ell=1}^L (L-\ell)g(n/2^{L-\ell})2^{L-\ell}
        \right).
    \]
    Since $g(n) = O(n \log n)$ \cite{pan2012structured}, we reach
    \[
        O(kn^2 + k^2 n ( L^2 \log n - L^3)).
    \]
    If, instead, a $g(n) = O(n^2)$ algorithm is used, the complexity of the merging step becomes $O(k^2 n^2)$, with a quadratic relationship on $k$. In practice, even if we opt for the $g(x) = O(x^2)$ algorithm, we have observed that the complexity of the update is closer to $O((L-\ell)g(n/2^{L-\ell}))$ (since not all the edges to update lay always on the same subgraph), which yields again $O(kn^2)$ behavior.
\end{itemize}

Thus, the full complexity of the algorithm is
\[
    O\!\left(kn^2 + \sum_{i = 1}^m \, f_i(n)\right)
\]

\paragraph{Parallel solver.}
Under parallel execution, with $m$ processors available, we focus on the parallel time $T_m(\cdot)$:

\begin{itemize}
    \item \textbf{Initialization.}
    Since each block operates on disjoint edges in $\mathcal{Z}$, operations 
    can be carried out in parallel. Assuming at least $m$ processors are available, by Brent's theorem \cite{jaja1992parallel}, we can compute the eigendecomposition in time
    \[
        T_m(n) = O\!\left( \max_{i = 1, \hdots, m} f_i(n) \right).
    \]

    \item \textbf{Hierarchical merging.}
    All subgraph pairs at level $\ell$ are independent, so the work at that level parallelizes fully. The work on each level is still governed by the cost of solving the secular equation, which has to be solved sequentially. Therefore, the time to compute this step remains in $T_m(n,k) = O(kn^2)$.
\end{itemize}

Thus, we can compute the parallel version in time
\[
    T_m(n, k) = O\!\left(k n^2 + \max_{i = 1, \hdots, m} f_i(n)\right).
\]

\begin{algorithm}[t]
\caption{Solve secular Equation}
\label{alg:solve_secular}
\begin{algorithmic}[1]
\Require 
    \State Old eigenvalues $\boldsymbol{\lambda} = (\lambda_1,\dots,\lambda_n)$, sorted ascending
    \State Projection vector $\mathbf{z} \in \mathbb{R}^n$
\Ensure 
    \State Updated eigenvalues $\tilde{\boldsymbol{\lambda}}$
    \State Cauchy matrix $\mathbf{C} \in \mathbb{R}^{n\times n}$ (columns are eigenvectors)

\Function{\texttt{SolveSecular}}{$\boldsymbol{\lambda},\mathbf{z}$}

    \State $n \gets \text{length}(\boldsymbol{\lambda})$
    \State Initialize $\tilde{\boldsymbol{\lambda}} \in \mathbb{R}^n$

    \For{$j = 1 \dots n$} 
        \State Solve for $\tilde{\lambda}_j$ in $(\lambda_j, \lambda_{j+1})$: $1 + \sum_{k=1}^n {z_k^2} \, / 
        \, (\tilde{\lambda}_j - \lambda_k) = 0$
    \EndFor

    \State Initialize $\mathbf{C} \in \mathbb{R}^{n\times n}$

    \For{$j = 1 \dots n$} \Comment{For each new eigenvalue (Column)}
        \For{$i = 1 \dots n$} \Comment{For each component (Row)}
            \State $C_{ij} \gets {z_i} / ({\tilde{\lambda}_j - \lambda_i})$
        \EndFor
        \State Normalize column $j$: $\mathbf{C}_{:,j} \gets \mathbf{C}_{:,j} \big/ 
            \left\| \mathbf{C}_{:,j} \right\|_2$
    \EndFor

    \State \Return $\tilde{\boldsymbol{\lambda}},\mathbf{C}$

\EndFunction
\end{algorithmic}
\end{algorithm}

\begin{algorithm}[t]
\caption{Cauchy factorization}
\label{alg:cauchy_full}
\begin{algorithmic}[1]

\Require
    \State Hierarchical graph $\mathcal{G}\in\mathcal{F}(L, \lbrace \mathcal{G}_i\rbrace_{i = 1}^m, k)$
    \State Bridge edge sets $\mathcal{E}_{\ell,i}$ with corresponding weights $w(\cdot)$ for $i = 1, \hdots, 2^{L-\ell}$ and levels $\ell=1\dots L$.
\Statex
\Ensure
    \State Final eigenvalues $\boldsymbol{\lambda}\in\mathbb{R}^N$, Cauchy rotation history $\mathcal{H}=\{\mathcal{H}_1,\dots,\mathcal{H}_L\}$.

\Statex

\begin{tcolorbox}[breakable, enhanced, colback=tabblue!6, boxrule=0pt, 
                  left= 0pt, right=0pt, top=0pt, bottom=4pt,
                  boxsep=0pt] 

\State \textbf{function} {\texttt{ApplyTransform}}({$\mathcal{I},\, \mathbf{M},\, \mathcal{Z}$})
   \State \; \; \; $\mathcal{E}_{\mathrm{act}} \gets \{\, e=(u,v) \in \text{Keys}(\mathcal{Z}) \mid \{u,v\}\cap\mathcal{I}\neq\emptyset \,\}$ 
    \Comment{Update edge vectors touching node set $\mathcal{I}$}
    \State \; \; \; \textbf{for} {each $e\in \mathcal{E}_{\mathrm{act}}$}
        \State \; \; \; \; \; \; $\mathbf{z}_{\mathrm{loc}} \gets \operatorname{Extract}(\mathcal{Z}[e],\,\mathcal{I})$
        \State \; \; \; \; \; \; $\mathcal{Z}[e] \gets \operatorname{Insert}(\mathcal{Z}[e],\,\mathcal{I},\,\mathbf{M}^\top\, \mathbf{z}_{\mathrm{loc}})$
    \State \; \; \; \textbf{return} $\mathcal{Z}$
\end{tcolorbox}
\begin{tcolorbox}[breakable, enhanced, colback=taborange!6, boxrule=0pt, 
                  left= 0pt, right=0pt, top=0pt, bottom=4pt,
                  boxsep=0pt] 
\Statex             
\State \textbf{1. Initialization step}
\State $\mathcal{E}_{\mathrm{total}} \gets \bigcup_{\ell,i}\, \mathcal{E}_{\ell,i}$
\State Initialize $\mathcal{Z}$ as map: edge $\mapsto$ sparse length-$N$ vector

    \For{each $e=(u,v)\in \mathcal{E}_{\mathrm{total}}$}
        \State $\mathcal{Z}[e] \gets \sqrt{w(e)}\, (\mathbf{e}_u - \mathbf{e}_v)$ 
    \EndFor

    \State $\mathcal{S}_{\mathrm{curr}} \gets \text{empty list}$

    \textbf{for} {$i=1\dots m$}
        \State \; \; \;  $(\mathbf{U}_i,\, \boldsymbol{\lambda}_i) \gets \Call{\texttt{DenseEigsLap}}{\mathcal{G}_i}$
        \State \; \; \; $\mathcal{I}_i \gets$ sorted node set of leaf $\mathcal{G}_i$
        \State \; \; \; $\mathcal{Z} \gets \Call{\texttt{ApplyTransform}}{\mathcal{I}_i,\, \mathbf{U}_i,\, \mathcal{Z}}$
        \State \; \; \; Append $(\boldsymbol{\lambda}_i,\, \mathcal{I}_i)$ to $\mathcal{S}_{\mathrm{curr}}$
    
\end{tcolorbox}

%

\Statex

\begin{tcolorbox}[breakable, enhanced, colback=tabpurple!6, boxrule=0pt, 
                  left= 0pt, right=0pt, top=0pt, bottom=4pt,
                  boxsep=0pt, left skip=0.0em]

\State \textbf{2. Hierarchical merging}
\For{$\ell = 1\dots L$}
    \State $\mathcal{S}_{\mathrm{next}} \gets \text{empty list}$
    \State $\mathcal{H}_\ell \gets \emptyset$

    \For{$i=1\dots |\mathcal{S}_{\mathrm{curr}}|/2$}
        \State Retrieve left $(\boldsymbol{\lambda}_A,\mathcal{I}_A) \gets \mathcal{S}_{\mathrm{curr}}[2i-1]$
        \State Retrieve right $(\boldsymbol{\lambda}_B,\mathcal{I}_B) \gets \mathcal{S}_{\mathrm{curr}}[2i]$
        
        \State $\mathcal{I}_{\mathrm{new}} \gets \mathcal{I}_A \cup \mathcal{I}_B$
        \State $\boldsymbol{\lambda}_{\mathrm{new}} \gets [\,\boldsymbol{\lambda}_A;\,\boldsymbol{\lambda}_B\,]$

        \For{each bridge $e\in \mathcal{E}_{\ell,k}$}
            \State $\mathbf{z} \gets \operatorname{Extract}(\mathcal{Z}[e],\, \mathcal{I}_{\mathrm{new}})$
            \State $(\boldsymbol{\lambda}_{\mathrm{new}},\, \mathbf{C}) \gets \Call{\texttt{SolveSecular}}{\boldsymbol{\lambda}_{\mathrm{new}},\, \mathbf{z}}$
            \State $\mathcal{Z} \gets \Call{\texttt{ApplyTransform}}{\mathcal{I}_{\mathrm{new}},\, \mathbf{C},\, \mathcal{Z}}$
            \State $\mathcal{H}_\ell \gets \mathcal{H}_\ell \cup \{\mathbf{C}\}$
        \EndFor

        \State Append $(\boldsymbol{\lambda}_{\mathrm{new}},\, \mathcal{I}_{\mathrm{new}})$ to $\mathcal{S}_{\mathrm{next}}$
    \EndFor

    \State $\mathcal{S}_{\mathrm{curr}} \gets \mathcal{S}_{\mathrm{next}}$
\EndFor
\end{tcolorbox}
\State \hspace{-1.5em}\Return final $\boldsymbol{\lambda}$ from  $\mathcal{S}_{\mathrm{curr}}$ and $\mathcal{H}$ 

\end{algorithmic}
\end{algorithm}

\section{Proof of \cref{thm:expressiveness}}\label{app:expressiveness}

We establish the result in three steps: we first show that polynomial filters and global spectral filters coincide on finite graphs, then prove containment, and finally exhibit a strict separation.

\subsection{Polynomial filters and global spectral filters coincide on finite graphs}\label{app:poly_equivalence}

A polynomial filter of degree $K$ computes
\begin{equation}
    \sum_{k=0}^{K} \alpha_k \mathbf{L}^k \mathbf{x} = \mathbf{U}\, p(\mathbf{\Lambda})\, \mathbf{U}^\top \mathbf{x},
\end{equation}
where $p(\lambda) = \sum_k \alpha_k \lambda^k$. Conversely, any global spectral filter $\mathbf{U}\, g(\mathbf{\Lambda})\, \mathbf{U}^\top$ for arbitrary $g(\cdot)$ can be expressed as a polynomial in $\mathbf{L}$ of degree at most $n - 1$. This follows from two observations:
\begin{enumerate}
    \item \textbf{Cayley--Hamilton:} $\mathbf{L}$ satisfies its own characteristic polynomial of degree $n$, so the algebra generated by $\mathbf{L}$ is spanned by $\{\mathbf{I}, \mathbf{L}, \ldots, \mathbf{L}^{n-1}\}$.
    \item \textbf{Lagrange interpolation:} On a fixed graph with $m \leq n$ distinct eigenvalues, any function $g(\cdot)$ on the spectrum $\{\lambda_1, \ldots, \lambda_m\}$ can be interpolated by a polynomial of degree at most $m - 1 \leq n - 1$~\citep{wang2022powerful}.
\end{enumerate}
Therefore, on a fixed finite graph, the two classes coincide:
\begin{equation}
    \bigl\{\mathbf{U}\, g(\mathbf{\Lambda})\, \mathbf{U}^\top : g(\cdot) \text{ arbitrary}\bigr\} = \bigl\{p(\mathbf{L}) : p(\cdot) \text{ polynomial of degree } \leq n-1\bigr\}.
\end{equation}
We use ``$g(\mathbf{L})$'' to denote this class throughout.

\subsection{Containment: $g(\mathbf{L}) \subseteq$ L2G-Net}\label{app:containment}

Consider a graph $\mathcal{G} \in \mathcal{F}(L, \{\mathcal{G}_i\}_{i=1}^m, k)$. By Theorem~\ref{thm:cauchy_factorization}, the GFT admits the Cauchy factorization
\begin{equation}
    \mathbf{U}^\top = \mathbf{D}(\boldsymbol{\lambda}, \tilde{\boldsymbol{\lambda}}_{K-1}) \cdots \mathbf{D}(\tilde{\boldsymbol{\lambda}}_1, \tilde{\boldsymbol{\lambda}}_0)\, \mathbf{U}_0^\top,
\end{equation}
where $\mathbf{U}_0 = \mathrm{blkdiag}(\mathbf{U}_1, \ldots, \mathbf{U}_m)$ collects the base subgraph GFTs and each $\mathbf{D}$ is a Cauchy factor.

In L2G-Net, at each hierarchical level $r$ and subgraph pair $p$, a learnable spectral filter $g_{r,p}(\boldsymbol{\lambda}_{r,p})$ is applied before the Cauchy merge. Setting all intermediate filters to the identity, i.e., $g_{r,p}(\lambda) = 1$ for all $\lambda$, all $r = 0, \ldots, L-1$, and all $p$, the forward transform $\mathbf{U}(\Phi)^\top$ reduces exactly to $\mathbf{U}^\top$ by Theorem~\ref{thm:cauchy_factorization}. The L2G-Net output then becomes
\begin{equation}
    \mathbf{X}_{\mathrm{out}} = \mathbf{U}\, g_\theta(\mathbf{\Lambda})\, \mathbf{U}^\top \mathbf{X},
\end{equation}
which is a standard global spectral filter. Since the identity is within the parameterization of L2G-Net's local filters, every global spectral filter is realizable. \qed

\subsection{Strict containment: $g(\mathbf{L}) \subsetneq$ L2G-Net}\label{app:strict}

We show that L2G-Net can represent operators that no global spectral filter $g(\mathbf{L})$ can express, by exhibiting an explicit counterexample.

\paragraph{General necessary condition.}
Consider a graph $\mathcal{G} \in \mathcal{F}(1, \{\mathcal{G}_1, \mathcal{G}_2\}, k)$ with $k \geq 1$, i.e., two subgraphs connected by at least one bridge edge. With one level of hierarchy, the L2G-Net forward transform is
\begin{equation}\label{eq:l2g_one_level}
    \mathbf{U}(\Phi)^\top = g_1(\mathbf{\Lambda}) \, \mathbf{D} \,  \mathrm{blkdiag}\bigl(g_{0,1}(\mathbf{\Lambda}_1),\; g_{0,2}(\mathbf{\Lambda}_2)\bigr) \mathbf{U}_0^\top,
\end{equation}
where $\mathbf{D}$ is the product of Cauchy factors for the bridge edges and $g_1(\cdot)$ acts on the merged spectrum. Define
\begin{equation}
    \mathbf{F}_{\mathrm{loc}} \doteq \mathrm{blkdiag}\bigl(g_{0,1}(\mathbf{\Lambda}_1),\; g_{0,2}(\mathbf{\Lambda}_2)\bigr).
\end{equation}
The full L2G-Net operator is
\begin{equation}
    \mathbf{T}_{\mathrm{L2G}} = \mathbf{U}\, g_\theta(\mathbf{\Lambda})\, g_1(\mathbf{\Lambda})\, \mathbf{D}\, \mathbf{F}_{\mathrm{loc}}\, \mathbf{U}_0^\top,
\end{equation}
while a global spectral filter takes the form
\begin{equation}
    \mathbf{T}_{\mathrm{GFT}} = \mathbf{U}\, h(\mathbf{\Lambda})\, \mathbf{U}^\top = \mathbf{U}\, h(\mathbf{\Lambda})\, \mathbf{D}\, \mathbf{U}_0^\top,
\end{equation}
where the last equality uses Theorem~\ref{thm:cauchy_factorization} with identity local filters.

For $\mathbf{T}_{\mathrm{L2G}} = \mathbf{T}_{\mathrm{GFT}}$ to hold for some $h(\cdot)$, we require
\begin{equation}
    g_\theta(\mathbf{\Lambda})\, g_1(\mathbf{\Lambda})\, \mathbf{D}\, \mathbf{F}_{\mathrm{loc}} = h(\mathbf{\Lambda})\, \mathbf{D}.
\end{equation}
Let $f(\mathbf{\Lambda}) := g_\theta(\mathbf{\Lambda})\, g_1(\mathbf{\Lambda})$ (a diagonal matrix). Since $\mathbf{D}$ is orthogonal (being a product of orthogonal Cauchy-like matrices), right-multiplying by $\mathbf{D}^\top$ yields
\begin{equation}\label{eq:necessary}
    f(\mathbf{\Lambda})\, \mathbf{D}\, \mathbf{F}_{\mathrm{loc}}\, \mathbf{D}^\top = h(\mathbf{\Lambda}).
\end{equation}
Since $h(\mathbf{\Lambda})$ is diagonal, a necessary condition for the L2G-Net operator to collapse to a global spectral filter is
\begin{equation}\label{eq:star}
    \mathbf{D}\, \mathbf{F}_{\mathrm{loc}}\, \mathbf{D}^\top \text{ is diagonal.} \tag{$\star$}
\end{equation}

\paragraph{Counterexample: two singleton subgraphs.}
Consider $n_1 = n_2 = 1$: two isolated nodes connected by a single edge of weight $w > 0$. The base subgraphs are singletons with trivial GFTs ($\mathbf{U}_0 = \mathbf{I}_2$). The Laplacian of the disconnected graph is $\mathbf{L}_0 = \mathbf{0}$, and adding the bridge edge gives
\begin{equation}
    \mathbf{L} = w \begin{pmatrix} 1 & -1 \\ -1 & 1 \end{pmatrix},
\end{equation}
with eigenvalues $0$ and $2w$, and eigenvectors
\begin{equation}
    \mathbf{U} = \frac{1}{\sqrt{2}} \begin{pmatrix} 1 & 1 \\ 1 & -1 \end{pmatrix}.
\end{equation}
By Theorem~\ref{thm:cauchy_factorization}, $\mathbf{U}^\top = \mathbf{D} \cdot \mathbf{U}_0^\top = \mathbf{D} \cdot \mathbf{I}_2$, so the Cauchy factor is
\begin{equation}
    \mathbf{D} = \mathbf{U}^\top = \frac{1}{\sqrt{2}} \begin{pmatrix} 1 & 1 \\ 1 & -1 \end{pmatrix}.
\end{equation}
This is a nontrivial orthogonal matrix (neither diagonal nor a permutation). Now let $\mathbf{F}_{\mathrm{loc}} = \mathrm{diag}(a, b)$ with $a \neq b$. Then
\begin{equation}
    \mathbf{D}\, \mathbf{F}_{\mathrm{loc}}\, \mathbf{D}^\top = \frac{1}{2} \begin{pmatrix} 1 & 1 \\ 1 & -1 \end{pmatrix} \begin{pmatrix} a & 0 \\ 0 & b \end{pmatrix} \begin{pmatrix} 1 & 1 \\ 1 & -1 \end{pmatrix} = \frac{1}{2} \begin{pmatrix} a+b & a-b \\ a-b & a+b \end{pmatrix}.
\end{equation}
The off-diagonal entry is $\frac{a-b}{2} \neq 0$ whenever $a \neq b$. Therefore, condition~\eqref{eq:star} fails, and the L2G-Net operator $\mathbf{T}_{\mathrm{L2G}}$ cannot be expressed as any global spectral filter $g(\mathbf{L})$.

Since this construction uses a valid member of $\mathcal{F}(1, \{\mathcal{G}_1, \mathcal{G}_2\}, 1)$ and a non-constant local filter (which is within L2G-Net's parameterization), the inclusion is strict. \qed

\begin{remark}[Generality beyond the counterexample]\label{rem:generality}
The two-node example is a minimal proof of strict inclusion. The phenomenon extends to arbitrary graphs: since Cauchy factors arising from bridge edges are dense orthogonal matrices with all nonzero entries, the conjugation $\mathbf{D}\, \mathbf{F}_{\mathrm{loc}}\, \mathbf{D}^\top$ is non-diagonal whenever $\mathbf{F}_{\mathrm{loc}}$ is not proportional to the identity.
\end{remark}

\section{L2G-Net intuition}
Given two dense subgraphs connected by a bridge edge (e.g., barbell-like graph \cite{topping2022understanding}, as shown in \cref{fig:barbell}), MPNNs require many aggregation steps to share
information between subgraphs, which may lead to oversquashing \cite{alon2020bottleneck}. In L2G-Net, the graph is partitioned into subgraphs, which are first processed separately, and then the Cauchy factor operates across the bridge edge to mix information from the subgraphs, enabling long-range communication with a single layer. Unlike the GFT, where all spectral components
are global, L2G-Net can learn different ways to combine information from subgraphs for each spectral component. 

\begin{figure}
    \centering
    \includegraphics[width=\linewidth]{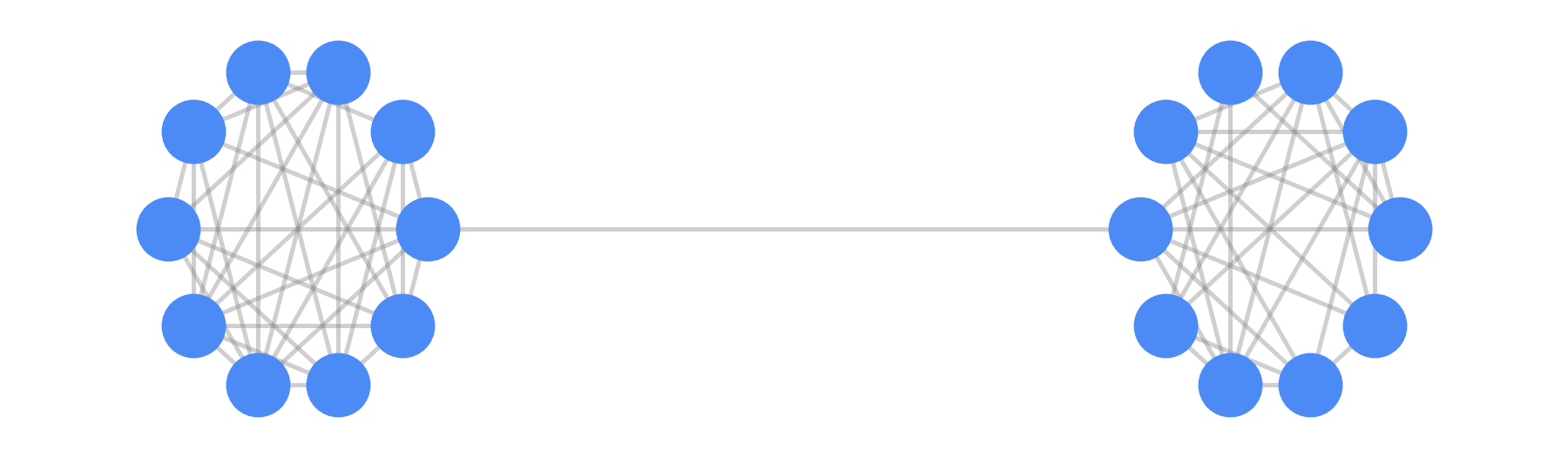}
    \caption{Example of a barbell-like graph.}
    \label{fig:barbell}
\end{figure}

\section{More experimental results}
\label{app:more_exper}
In \cref{tab:table9}, we compare our method with a benchmark based on \cite{platonov2023critical}. In \cref{tab:lrgbench_long} we compare the results in the peptides datasets of \cite{dwivedi2022long}. In \cref{tab:citynetwork_long}, we compare the results in the City-Networks dataset of \cite{liang2025towards}.

\subsection{Learned filters}
\label{app:learned_filters}
We show the learned spectral filters for different graphs in \cite{platonov2023critical} (\cref{fig:learned_filters}). We consider only a two-level decomposition. We also add the histogram of the distribution of the eigenvalues in the background, for each subgraph and for the global graph. We observe that hierarchical processing adapts to the task, which contributes to the interpretability of our method.

\begin{figure}
    \raggedright
    (a) \emph{Amazon-ratings}

    \smallskip
    
    \includegraphics[width=\linewidth]{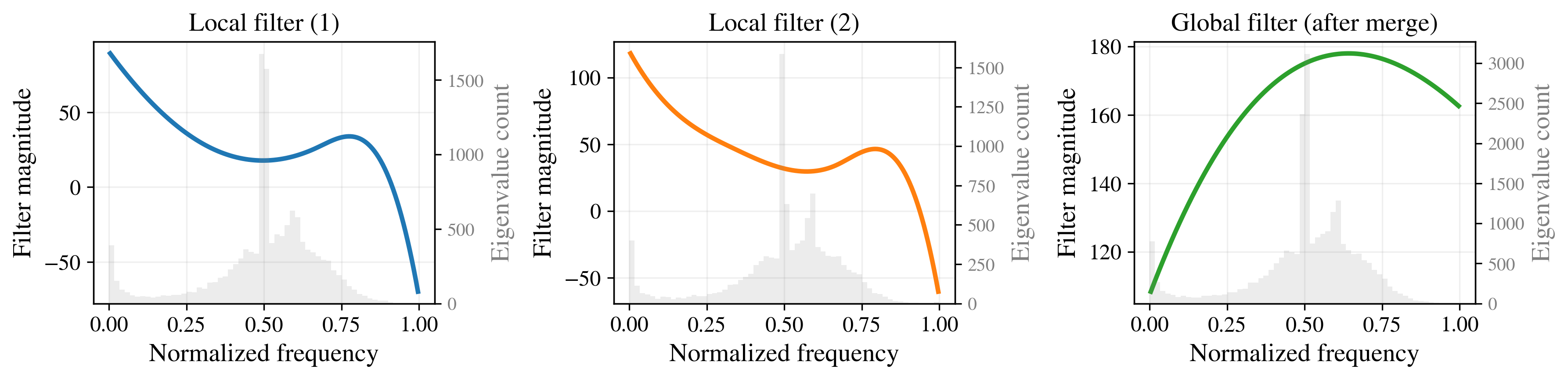}
    
    (b) \emph{Roman-empire}

    \smallskip
    
    \includegraphics[width=\linewidth]{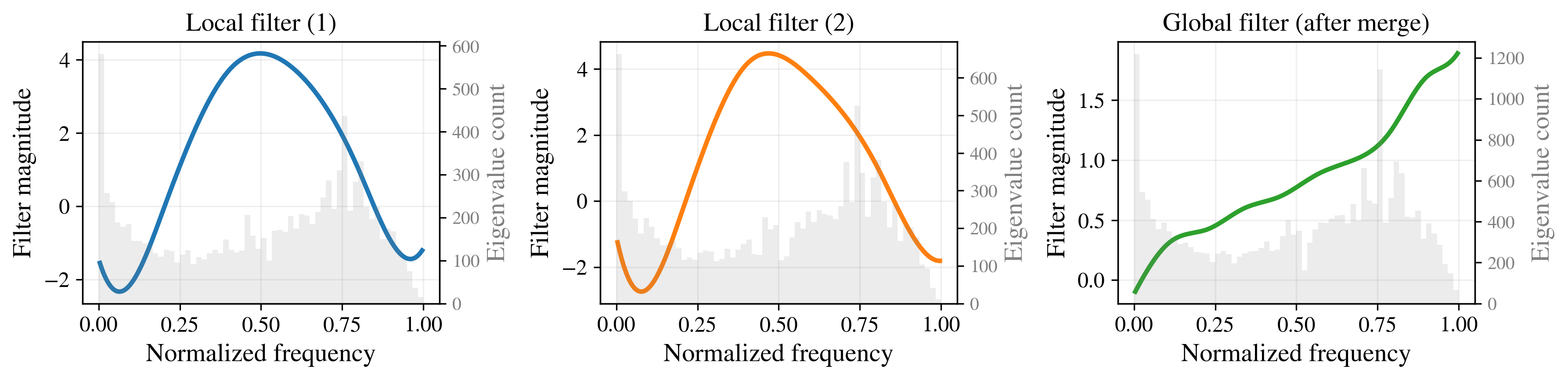}
    
    (c) \emph{Minesweeper}

    \smallskip
    
    \includegraphics[width=\linewidth]{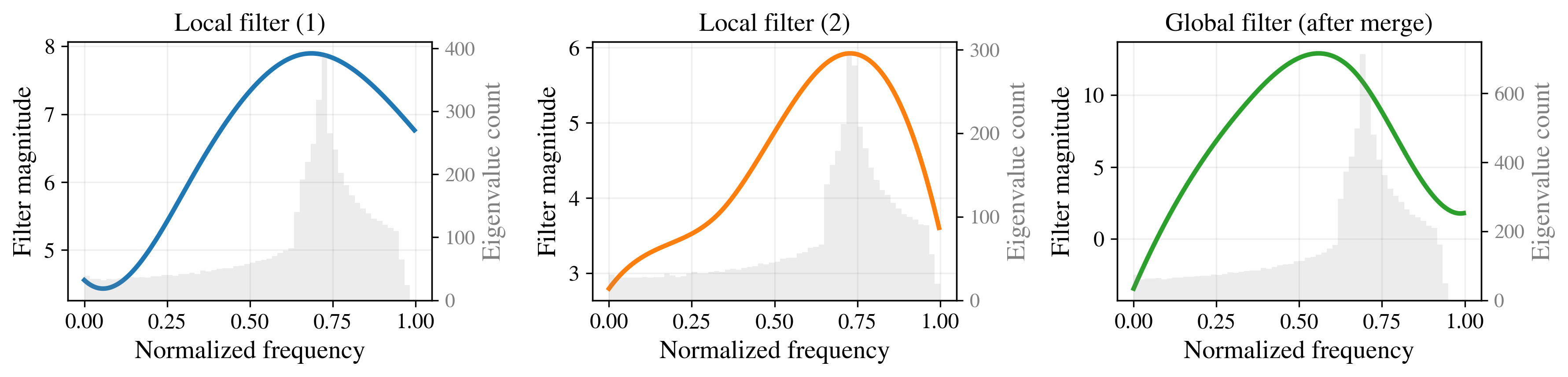}
    
    (d) \emph{Tolokers}

    \smallskip
    
    \includegraphics[width=\linewidth]{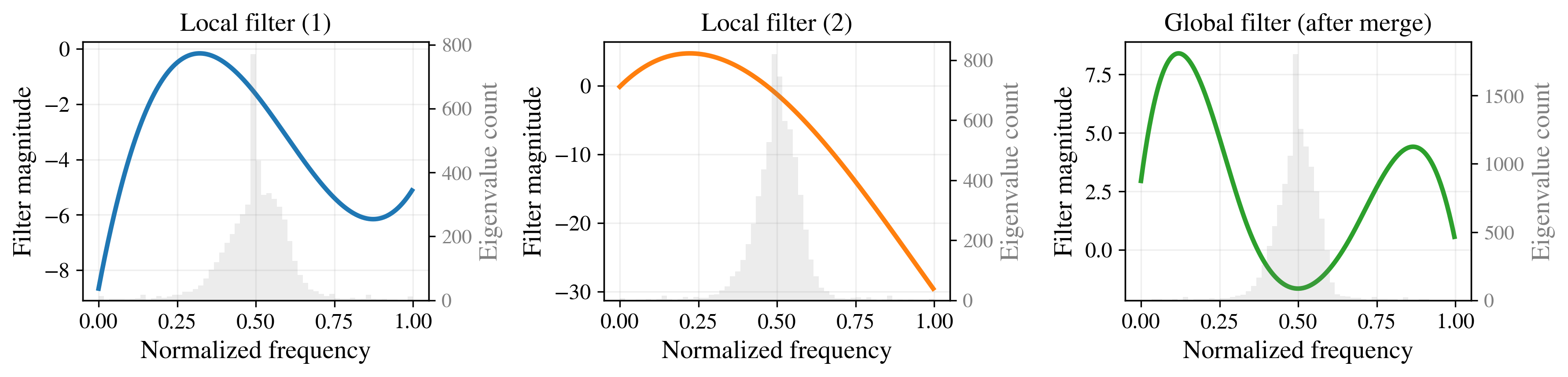}
    \caption{Learned spectral filters for different graphs in \cite{platonov2023critical}. In tasks where both low- and high-frequency information are needed (e.g., \emph{Tolokers}/\emph{Amazon-ratings}), local filters emphasize low frequencies, while global filters focus on higher frequencies. When the signal lies in high frequencies
(\emph{Minesweeper}/\emph{Roman-empire}), both local and global filters focus on high-frequency components. This indicates that the hierarchical processing adapts to the task.}
    \label{fig:learned_filters}
\end{figure}

\begin{figure}[t]
    \centering
\includegraphics[width=0.3\linewidth]{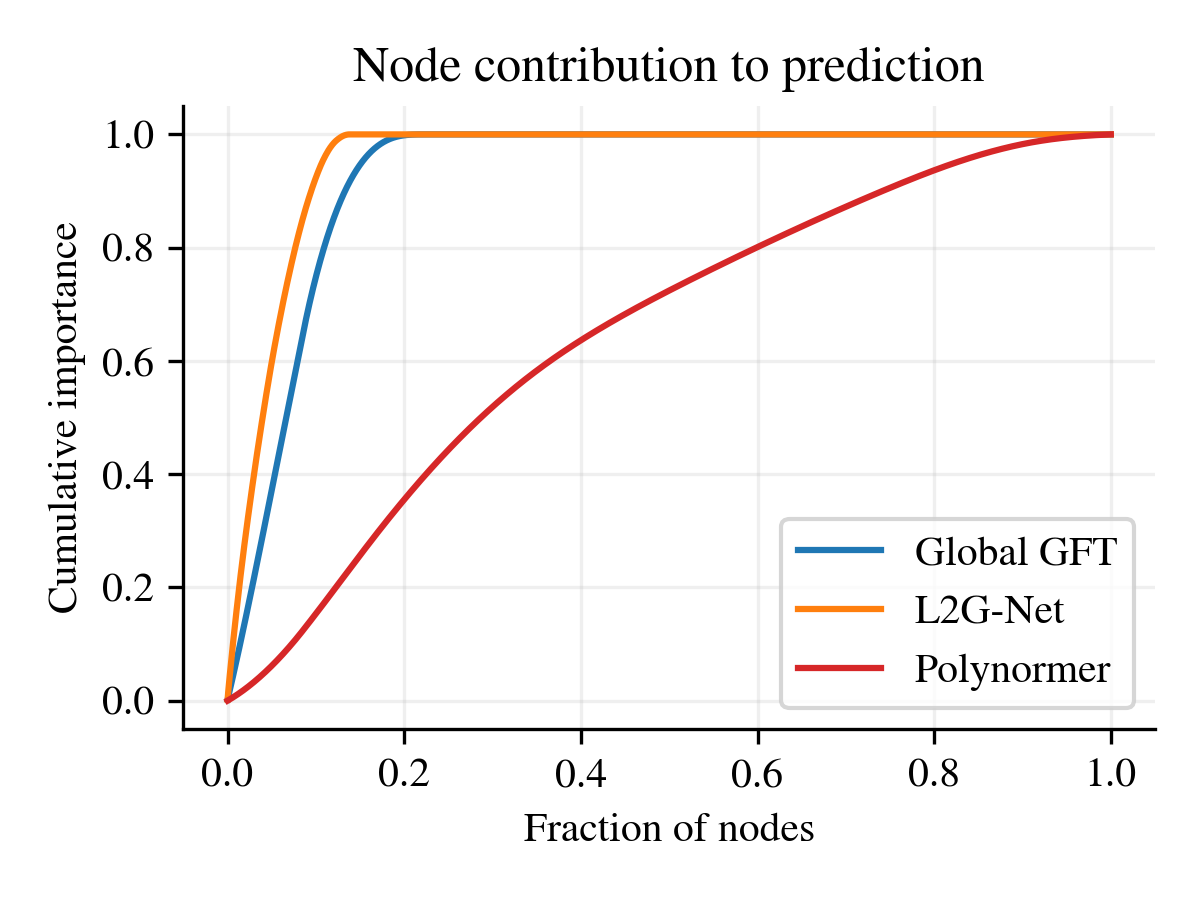}  
\includegraphics[width=0.3\linewidth]{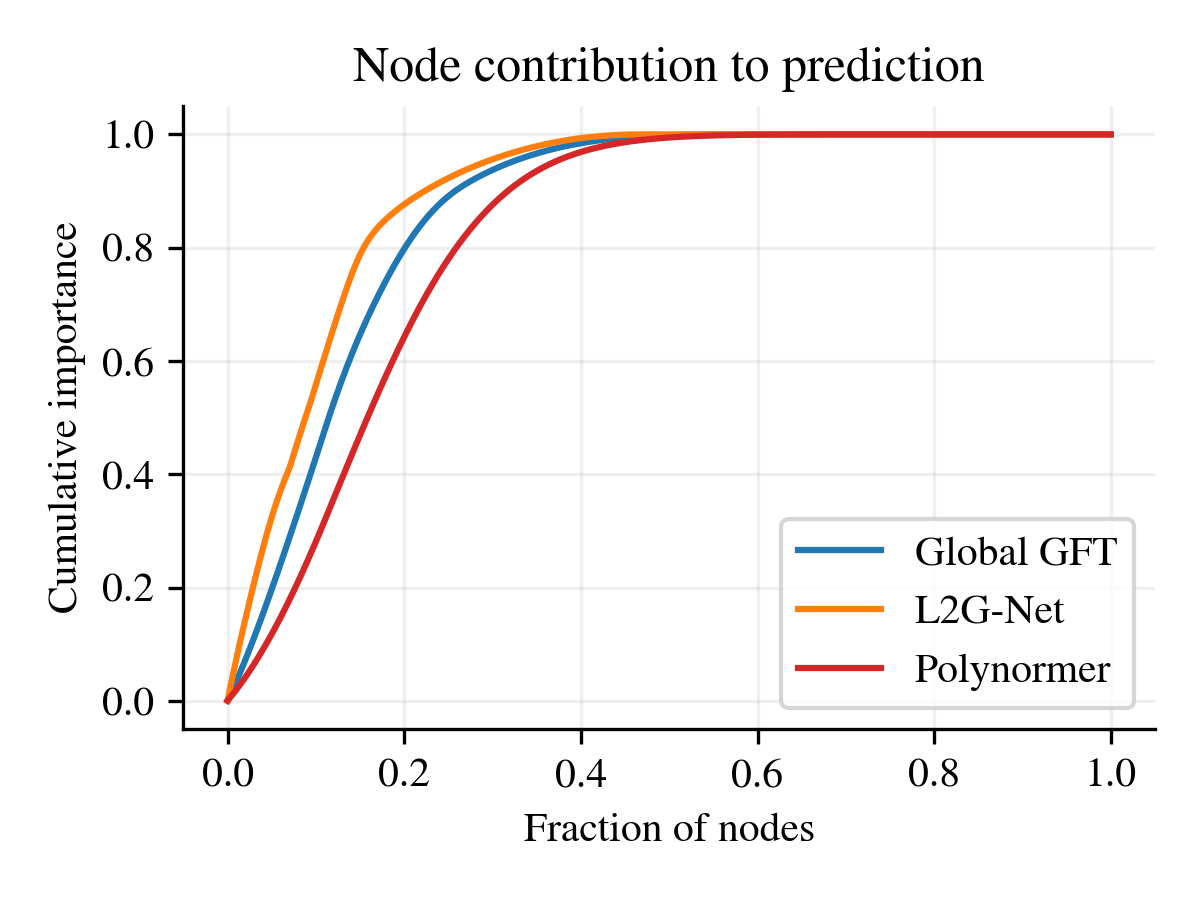}
\includegraphics[width=0.3\linewidth]{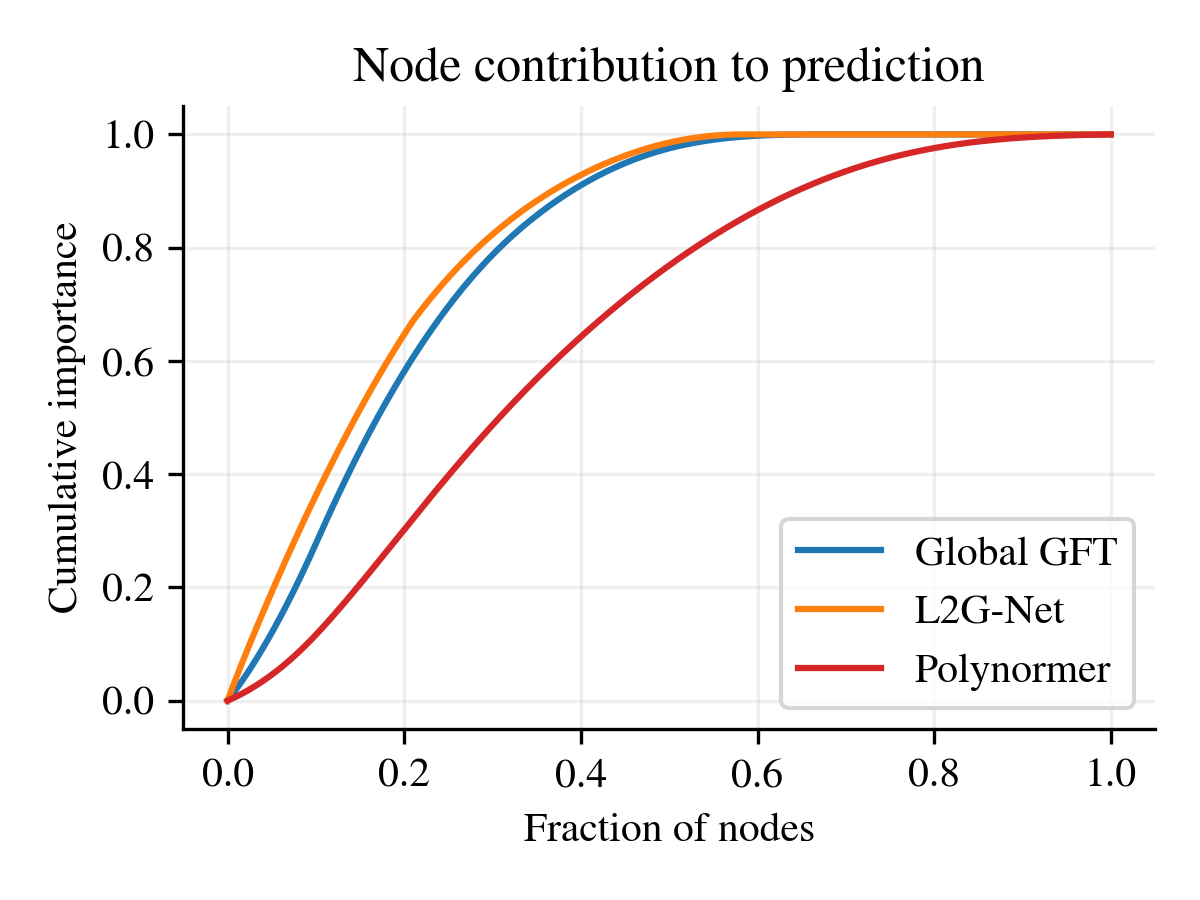}  

    \caption{\chan{Cumulative node contribution to the prediction on \textit{Tolokers, Roman Empire}, and \textit{Amazon Ratings} (from left to right). Average across validation set; shaded areas indicate confidence intervals. L2G-Net consistently concentrates predictive importance on a smaller fraction of nodes than both the Global GFT and Polynormer, showing consistent behavior across datasets.}}
    \label{fig:gradcam_app}
\end{figure}

\begin{table*}[t]
\centering
\caption{End-to-end training time (mm:ss) to convergence on heterophilous benchmarks~\cite{platonov2023critical}, under identical hardware.}
\label{tab:training_time}
\setlength{\tabcolsep}{8pt}
\begin{tabular}{lcccc}
\toprule
Method & \textbf{Mines.} & \textbf{Tolok.} & \textbf{Am. Rat.} & \textbf{R. Emp.} \\
\midrule
GCN         & 01:23 & 01:53 & 02:12 & 02:02 \\
ChebNet     & 03:56 & 06:18 & 11:22 & 13:08 \\
Polynormer  & 15:12 & 08:32 & 35:00 & 28:56 \\
L2G-Net (ours) & 08:40 & 03:04 & 34:56 & 50:03 \\
\bottomrule
\end{tabular}
\end{table*}

\begin{table*}[t]
\centering
\caption{Per-epoch wall-clock time (s/epoch) on heterophilous benchmarks~\cite{platonov2023critical}.}
\label{tab:per_epoch}
\setlength{\tabcolsep}{8pt}
\begin{tabular}{lcccc}
\toprule
Method & \textbf{Mines.} & \textbf{Tolok.} & \textbf{Am. Rat.} & \textbf{R. Emp.} \\
\midrule
GCN         & 0.05 & 0.06 & 0.07 & 0.06 \\
ChebNet     & 0.12 & 0.21 & 0.33 & 0.40 \\
Polynormer  & 0.29 & 0.18 & 0.56 & 0.48 \\
L2G-Net (ours) & 0.35 & 0.13 & 1.04 & 1.54 \\
\bottomrule
\end{tabular}
\end{table*}

\begin{table*}[t]
\centering
\caption{Performance comparison on long-range benchmarks. Results are reported as mean $\pm$ standard deviation. We report Acc for \textit{Roman-empire} and \textit{Amazon-ratings} and AUC for \textit{Tolokers} and \textit{Minesweeper}. Higher is better ($\uparrow$). Best method appears in boldface, second best method underlined, third best method in italics.}
\label{tab:table9}
\resizebox{0.9\textwidth}{!}{
\begin{tabular}{lcccc}
\toprule
\textbf{Model} & \textbf{Roman-empire} & \textbf{Amazon-ratings} & \textbf{Minesweeper} & \textbf{Tolokers} \\
\midrule

\multicolumn{5}{l}{\textit{MPNNs}} \\
GAT & 80.87$\pm$0.30 & 49.09$\pm$0.63 & 92.01$\pm$0.68 & 83.70$\pm$0.47 \\
GAT (LapPE) & 84.80$\pm$0.46 & 44.90$\pm$0.73 & 93.50$\pm$0.54 & 84.99$\pm$0.54 \\
GAT (RWSE) & 86.62$\pm$0.53 & 48.58$\pm$0.41 & 92.53$\pm$0.65 & 85.02$\pm$0.67 \\
Gated-GCN & 74.46$\pm$0.54 & 43.00$\pm$0.32 & 87.54$\pm$1.22 & 77.31$\pm$1.14 \\
GCN & 73.69$\pm$0.74 & 48.70$\pm$0.63 & 89.75$\pm$0.52 & 83.64$\pm$0.67 \\
GCN (LapPE) & 83.37$\pm$0.55 & 44.35$\pm$0.36 & 94.26$\pm$0.49 & 84.95$\pm$0.78 \\
GCN (RWSE) & 84.84$\pm$0.55 & 46.40$\pm$0.55 & 93.84$\pm$0.48 & 85.11$\pm$0.77 \\
CO-GNN ($\Sigma,\Sigma$) & 91.57$\pm$0.32 & 51.28$\pm$0.56 & 95.09$\pm$1.18 & 83.36$\pm$0.89 \\
CO-GNN ($\mu,\mu$) & 91.37$\pm$0.35 & \underline{54.17$\pm$0.37} & 97.31$\pm$0.41 & 84.45$\pm$1.17 \\
SAGE & 85.74$\pm$0.67 & 53.63$\pm$0.39 & 93.51$\pm$0.57 & 82.43$\pm$0.44 \\

\midrule
\multicolumn{5}{l}{\textit{Graph Transformers}} \\
Exphormer & 89.03$\pm$0.37 & 53.51$\pm$0.46 & 90.74$\pm$0.53 & 83.77$\pm$0.78 \\
NAGphormer & 74.34$\pm$0.77 & 51.26$\pm$0.72 & 84.19$\pm$0.66 & 78.32$\pm$0.95 \\
GOAT & 71.59$\pm$1.25 & 44.61$\pm$0.50 & 81.09$\pm$1.02 & 83.11$\pm$1.04 \\
GPS & 82.00$\pm$0.61 & 53.10$\pm$0.42 & 90.63$\pm$0.67 & 83.71$\pm$0.48 \\
GPSGCN+Performer (LapPE) & 83.96$\pm$0.53 & 48.20$\pm$0.67 & 93.85$\pm$0.41 & 84.72$\pm$0.77 \\
GPSGCN+Performer (RWSE) & 84.72$\pm$0.65 & 48.08$\pm$0.85 & 92.88$\pm$0.50 & 84.81$\pm$0.86 \\
GPSGCN+Transformer (LapPE) & OOM & OOM & 91.82$\pm$0.41 & 83.51$\pm$0.93 \\
GPSGCN+Transformer (RWSE) & OOM & OOM & 91.17$\pm$0.51 & 83.53$\pm$1.06 \\
GT & 86.51$\pm$0.73 & 51.17$\pm$0.66 & 91.85$\pm$0.76 & 83.23$\pm$0.64 \\
GT-sep & 87.32$\pm$0.39 & 52.18$\pm$0.80 & 92.29$\pm$0.47 & 82.52$\pm$0.92 \\
Polynormer & \textbf{92.55$\pm$0.30} & \textbf{54.81$\pm$0.49} & \underline{97.46$\pm$0.36} & \underline{85.91$\pm$0.74} \\

\midrule
\multicolumn{5}{l}{\textit{Heterophily-Designated GNNs}} \\
CPGNN & 63.96$\pm$0.62 & 39.79$\pm$0.77 & 52.03$\pm$5.46 & 73.36$\pm$1.01 \\
FAGCN & 65.22$\pm$0.56 & 44.12$\pm$0.30 & 88.17$\pm$0.73 & 77.75$\pm$1.05 \\
FSGNN & 79.92$\pm$0.56 & 52.74$\pm$0.83 & 90.08$\pm$0.70 & 82.76$\pm$0.61 \\
GBK-GNN & 74.57$\pm$0.47 & 45.98$\pm$0.71 & 90.85$\pm$0.58 & 81.01$\pm$0.67 \\
GloGNN & 59.63$\pm$0.69 & 36.89$\pm$0.14 & 51.08$\pm$1.23 & 73.39$\pm$1.17 \\
GPR-GNN & 64.85$\pm$0.27 & 44.88$\pm$0.34 & 86.24$\pm$0.61 & 72.94$\pm$0.97 \\
H2GCN & 60.11$\pm$0.52 & 36.47$\pm$0.23 & 89.71$\pm$0.31 & 73.35$\pm$1.01 \\
JacobiConv & 71.14$\pm$0.42 & 43.55$\pm$0.48 & 89.66$\pm$0.40 & 68.66$\pm$0.65 \\

\midrule
\multicolumn{5}{l}{\textit{Graph SSMs}} \\
GMN & 87.69$\pm$0.50 & \emph{54.07$\pm$0.31} & 91.01$\pm$0.23 & 84.52$\pm$0.21 \\
GPS + Mamba & 83.10$\pm$0.28 & 45.13$\pm$0.97 & 89.93$\pm$0.54 & 83.70$\pm$1.05 \\
GRAMA$_{\tiny\mathrm{GCN}}$ & 88.61$\pm$0.43 & 53.48$\pm$0.62 & 95.27$\pm$0.71 & \textbf{86.23$\pm$1.10} \\
MP-SSM & 90.91$\pm$0.48 & 53.65$\pm$0.71 & 95.33$\pm$0.72 & 85.26$\pm$0.93 \\

\midrule
\multicolumn{5}{l}{\textit{Spectral}} \\
Stable-ChebNet & 92.03$\pm$0.85 & 53.15$\pm$0.21 & 95.71$\pm$2.26 & 85.55$\pm$3.35 \\
Ours & \underline{92.12$\pm$1.12} & 53.39$\pm$0.58 & \textbf{97.50$\pm$0.26} & \emph{85.57$\pm$0.61} \\
\bottomrule
\end{tabular}
}
\end{table*}

\begin{table*}[t]
\centering
\scriptsize
\setlength{\tabcolsep}{4.5pt}
\begin{tabular}{llcc}
\toprule
\textbf{Type} & \textbf{Model} & \textbf{peptides-func} (AP $\uparrow$) & \textbf{peptides-struct} (MAE $\downarrow$) \\
\midrule
\multirow{7}{*}{Trans.}
& SAN+LapPE     & 63.84 $\pm$ 1.21 & 0.2683 $\pm$ 0.0043 \\
& TIGT          & 66.79 $\pm$ 0.74 & 0.2485 $\pm$ 0.0015 \\
& Specformer    & 66.86 $\pm$ 0.64 & 0.2550 $\pm$ 0.0014 \\
& Exphormer     & 65.27 $\pm$ 0.43 & 0.2481 $\pm$ 0.0007 \\
& G.MLPMixer    & 69.21 $\pm$ 0.54 & 0.2475 $\pm$ 0.0015 \\
& Graph ViT     & 69.42 $\pm$ 0.75 & \emph{0.2449} $\pm$ 0.0016 \\
& GRIT          & 69.88 $\pm$ 0.82 & 0.2460 $\pm$ 0.0012 \\
\midrule
\multirow{3}{*}{Rewir.}
& LASER         & 64.40 $\pm$ 0.10 & 0.3043 $\pm$ 0.0019 \\
& DRew-GCN      & 69.96 $\pm$ 0.76 & 0.2781 $\pm$ 0.0028 \\
& \hspace{1em}+PE           & 71.50 $\pm$ 0.44 & 0.2536 $\pm$ 0.0015 \\
\midrule
\multirow{3}{*}{SS}
& Graph Mamba   & 67.39 $\pm$ 0.87 & 0.2478 $\pm$ 0.0016 \\
& GMN           & 70.71 $\pm$ 0.83 & 0.2473 $\pm$ 0.0025 \\
& MP-SSM        & 69.93 $\pm$ 0.52 & 0.2458 $\pm$ 0.0017 \\
\midrule
\multirow{14}{*}{GNN}
& A-DGN         & 59.75 $\pm$ 0.44 & 0.2874 $\pm$ 0.0021 \\
& ChebNet       & 69.61 $\pm$ 0.33 & 0.2627 $\pm$ 0.0033 \\
& ChebNetII     & 68.19 $\pm$ 0.27 & 0.2618 $\pm$ 0.0058 \\
& GCN           & 68.60 $\pm$ 0.50 & 0.2460 $\pm$ 0.0007 \\
& GRAMA         & 70.93 $\pm$ 0.78 & \textbf{0.2436} $\pm$ 0.0022 \\
& GRAND         & 57.89 $\pm$ 0.62 & 0.3418 $\pm$ 0.0015 \\
& GraphCON      & 60.22 $\pm$ 0.68 & 0.2778 $\pm$ 0.0018 \\
& PH-DGN        & 70.12 $\pm$ 0.45 & {0.2465} $\pm$ 0.0020 \\
& SWAN          & 67.51 $\pm$ 0.39 & 0.2485 $\pm$ 0.0009 \\
& PathNN        & 68.16 $\pm$ 0.26 & 0.2545 $\pm$ 0.0032 \\
& CIN++         & 65.69 $\pm$ 1.17 & 0.2523 $\pm$ 0.0013 \\
& S2GCN         & \textbf{72.75} $\pm$ 0.66 & \underline{0.2467} $\pm$ 0.0019 \\
& \hspace{1em}+PE           & \textbf{73.11} $\pm$ 0.66 & \underline{0.2447} $\pm$ 0.0032 \\
& Stable-ChebNet & 70.32 $\pm$ 0.26 & 0.2542 $\pm$ 0.0030 \\
& L2G-Net (ours) & \underline{72.14} $\pm$ 0.24 & 0.2479 $\pm$ 0.0012 \\
& \hspace{1em}+PE           & \underline{72.46} $\pm$ 0.25 & 0.2462 $\pm$ 0.0011 \\
\bottomrule
\end{tabular}

\medskip

\caption{Long-range benchmark results \cite{dwivedi2022long}. AP is reported on peptides-func (higher is better), and MAE on peptides-struct (lower is better). Trans. stands for transformer, Rewir. for rewiring, and SS for state space.}
\label{tab:lrgbench_long}
\end{table*}

\begin{table}[t]
\centering
\caption{Average test accuracy with standard deviation for $4$ random seeds on City-Networks \cite{liang2025towards}.}
\label{tab:citynetwork_long}
\begin{tabular}{lcccc}
\toprule
Model & \textbf{Paris} & \textbf{Shanghai} & \textbf{LA} & \textbf{London} \\
\midrule
MLP       & 25.5$_{(0.4)}$ & 48.4$_{(0.6)}$ & 24.1$_{(0.5)}$ & 27.9$_{(0.1)}$ \\
ChebNet   & 54.1$_{(0.2)}$ & 66.5$_{(0.1)}$ & 61.4$_{(0.4)}$ & 54.7$_{(0.2)}$ \\
GCN       & 53.2$_{(0.3)}$ & 62.1$_{(0.2)}$ & 58.3$_{(0.3)}$ & 50.1$_{(0.7)}$ \\
SAGE & 54.6$_{(0.2)}$ & 68.3$_{(0.5)}$ & 61.4$_{(0.3)}$ & 55.4$_{(0.2)}$ \\
GAT       & 51.1$_{(0.3)}$ & 68.0$_{(0.5)}$ & 59.5$_{(0.3)}$ & 52.0$_{(0.3)}$ \\
GCNII     & 51.3$_{(0.2)}$ & 61.5$_{(0.4)}$ & 56.0$_{(0.3)}$ & 48.2$_{(0.3)}$ \\
DropEdge  & 48.2$_{(0.2)}$ & 60.8$_{(0.4)}$ & 55.5$_{(0.3)}$ & 45.0$_{(0.3)}$ \\
\midrule
GraphGPS  & 52.1$_{(0.6)}$ & 63.0$_{(0.5)}$ & 59.8$_{(0.5)}$ & OOM \\
Exphormer & 55.1$_{(0.8)}$ & 70.2$_{(0.4)}$ & 63.8$_{(0.6)}$ & 49.5$_{(0.4)}$ \\
SGFormer  & 52.0$_{(0.8)}$ & 64.1$_{(0.3)}$ & 60.1$_{(0.7)}$ & 48.3$_{(0.3)}$ \\
\midrule
\textbf{Ours} & 55.4$_{(0.4)}$ & 69.8$_{(0.5)}$ & 63.5$_{(0.7)}$ & 55.6$_{(0.6)}$\\
\bottomrule
\end{tabular}
\end{table}

\subsection{GradCAM measurements}
\label{app:gradcam}
\chan{We show the GradCAM attribution curves for the three other graphs (\emph{Tolokers}, \emph{Roman Empire}, and \emph{Amazon Ratings}) in \cite{platonov2023critical} in \cref{fig:gradcam_app}. We observe that the same behavior we reported in \cref{fig:gradcam} generalizes to the other graphs in the suite.}

\subsection{Runtime Measurements}
\label{app:runtime}

In this section we provide a detailed empirical comparison of the runtime of L2G-Net against representative baselines. All experiments are run under identical hardware and training protocols. For L2G-Net, GCN, and ChebNet we match architecture hyperparameters (depth, hidden dimension); for Polynormer we use the configuration provided by the original authors, consistent with the setup in \cref{tab:main_results}. End-to-end training times for L2G-Net include the cost of computing the Cauchy factorization.

\paragraph{End-to-end training time.} \cref{tab:training_time} reports the total wall-clock time required to train each method to convergence on the heterophilous benchmarks of~\cite{platonov2023critical}. Although Polynormer attains a lower per-epoch cost than L2G-Net (see Table~\ref{tab:per_epoch}), its two-stage training procedure requires substantially more epochs to converge, so its overall training time is longer than that of L2G-Net on three of the four datasets. Compared to ChebNet, L2G-Net incurs a moderate overhead that reflects its richer subgraph-level spectral processing, but remains well within the same order of magnitude. Compared to a standard MPNN such as GCN, L2G-Net is slower; this is expected, as our method targets the regime of exact spectral processing rather than linear-time message passing.

\paragraph{Per-epoch wall-clock time.} \cref{tab:per_epoch} reports per-epoch wall-clock times in seconds. GCN is the fastest, consistent with its linear $\mathcal{O}(|E|)$ complexity. ChebNet and Polynormer occupy an intermediate range. L2G-Net is competitive on smaller graphs (\emph{Minesweeper}, \emph{Tolokers}) but its per-epoch cost grows with $n$, reflecting the $O(n^2)$ complexity of the Cauchy factorization-based forward pass. The trade-off is favorable in the end-to-end training picture because L2G-Net requires fewer epochs to converge than transformer baselines.

\end{document}